\documentclass[journal]{IEEEtran}
\usepackage{amsmath,amsfonts,amssymb,amsthm}
\usepackage{bm}
\usepackage{mathtools}
\usepackage{algorithmic}
\usepackage{algorithm}
\usepackage{enumitem}
\usepackage{color}

\usepackage{array}
\usepackage{booktabs}
\usepackage{threeparttable}
\usepackage{multirow}
\usepackage{textcomp}
\usepackage{stfloats}
\usepackage{verbatim}

\usepackage{tabularray}
\UseTblrLibrary{booktabs}
\usepackage{ninecolors}

\usepackage{url}
\usepackage{graphicx}
\usepackage{cite}

\ifCLASSOPTIONcompsoc
\usepackage[caption=false, font=normalsize, labelfont=sf, textfont=sf]{subfig}
\else
\usepackage[caption=false, font=footnotesize]{subfig}
\fi

\def\0{{\mathbf 0}}
\def\1{{\mathbf 1}}

\def\e{{\mathbf e}}
\def\f{{\mathbf f}}

\def\h{{\mathbf h}}

\def\s{{\mathbf s}}

\def\v{{\mathbf v}}

\def\x{{\mathbf x}}
\def\y{{\mathbf y}}

\def\B{{\mathbf B}}

\def\D{{\mathbf D}}

\def\H{{\mathbf H}}
\def\I{{\mathbf I}}

\def\L{{\mathbf L}}
\def\M{{\mathbf M}}

\def\P{{\mathbf P}}

\def\S{{\mathbf S}}

\def\W{{\mathbf W}}

\def\ie{{\textit{i.e.}}}

\def\wrt{{\textit{w.r.t.}}}

\def\cA{{\mathcal A}}

\def\cE{{\mathcal E}}

\def\cG{{\mathcal G}}

\def\cL{{\mathcal L}}

\def\cN{{\mathcal N}}
\def\cO{{\mathcal O}}

\def\cS{{\mathcal S}}

\def\cU{{\mathcal U}}
\def\cV{{\mathcal V}}

\newcommand{\diag}[1]{\text{diag}\mspace{-3mu}\left(#1\right)}

\theoremstyle{plain}

\newtheorem{lemma}{Lemma}

\theoremstyle{definition}

\theoremstyle{definition}

\newcommand{\gdas}{BS-GDA}

\newcommand{\ourdataset}{MHSVD}
\newcommand{\ourdatasetinfull}{Multi-Highlight Short Video Dataset}

\newcommand{\blue}[1]{#1}

\newtheorem{theorem}{Theorem}

\newtheorem{corollary}{Corollary}

\begin{document}

\title{Graph Unfolding and Sampling for Transitory Video Keyframe Selection via Gershgorin Disc Alignment}

\author{Sadid Sahami, Gene Cheung,~\IEEEmembership{Fellow, IEEE}, and Chia-Wen Lin,~\IEEEmembership{Fellow, IEEE} \thanks{Manuscript resubmitted on December 19, 2025. Chia-Wen Lin acknowledges the support of the National Council of Science and Technology, Taiwan, under Grants NSTC 112-2634-F-002-005 and NSTC 112-2221-E-007-077-MY3. Gene Cheung acknowledges the support of the NSERC grant RGPIN-2025-06252.}
	\thanks{S. Sahami is with the Institute of Communications Engineering, National Tsing Hua University, Hsinchu, Taiwan. (e-mail:s.sahami@ieee.org)}\thanks{Chia-Wen Lin is with the Department of Electrical Engineering and the Institute of Communications Engineering, National Tsing Hua University, Hsinchu, Taiwan. (e-mail: cwlin@ee.nthu.edu.tw)}
	\thanks{Gene Cheung is with York University, Toronto, Canada. (e-mail: genec@yorku.ca)}}

\markboth{IEEE TRANSACTIONS ON IMAGE PROCESSING,~Vol.~14, No.~8, June~2026}{S. Sahami \MakeLowercase{\textit{et al.}}: Graph Unfolding and Sampling for Transitory Video Keyframe Selection via Gershgorin Disc Alignment}

\maketitle
\begin{abstract}
User-generated videos (UGVs) uploaded from mobile phones to social media sites like YouTube and TikTok are short and non-repetitive.
We summarize a transitory UGV into several keyframes in linear-time via fast graph sampling based on Gershgorin disc alignment (GDA).
Specifically, we first model a sequence of $N$ frames in a UGV as an $M$-hop path graph $\cG^o$ for $M \ll N$, where the similarity between two frames within $M$ time instants is encoded as a positive edge based on feature similarity.
Towards efficient sampling, we then ``unfold'' $\cG^o$ to a $1$-hop path graph $\cG$, specified by a generalized graph Laplacian matrix $\cL$, via one of two graph unfolding procedures with provable performance bounds.
We show that maximizing the smallest eigenvalue $\lambda_{\min}(\B)$ of a coefficient matrix $\B = \diag{\h} + \mu \cL$, where $\h$ is the binary keyframe selection vector, is equivalent to minimizing a worst-case signal reconstruction error.
We maximize instead the Gershgorin circle theorem (GCT) lower bound $\lambda^-_{\min}(\B)$ by choosing $\h$ via a new fast graph sampling algorithm that iteratively aligns left-ends of Gershgorin discs for all graph nodes (frames).
Experiments on multiple short video datasets show that our algorithm achieves comparable or better keyframe selection performance compared to state-of-the-art methods, at a substantially reduced complexity.
\end{abstract}

\begin{IEEEkeywords}\raggedright
	keyframe extraction, video summarization, graph signal processing, Gershgorin circle theorem
\end{IEEEkeywords}

\section{Introduction}
\label{sec:intro}

The widespread use of mobile phones equipped with high-resolution cameras has led to a dramatic increase in the volume of video content generated.
For instance, by February 2020, YouTube users were uploading 500 hours of video content every minute.
Moreover, many platforms such as YouTube Shorts, TikTok, and Instagram Reels are dedicated to short-form content, which is typically no longer than three minutes, fast-moving, and non-repetitive~\cite{violot2024shorts,zhang2023measurement}.
Consequently, there is a critical need for automatic video summarization schemes that can condense each ``transitory'' video into a few representative keyframes quickly (ideally in linear-time), facilitating downstream tasks such as selection, retrieval~\cite{hu2011survey}, and classification~\cite{brezeale2008automatic}.

Although keyframe extraction has been extensively studied, many existing algorithms suffer from high computational complexity.
For instance, \cite{mundur2006keyframebased} uses \textit{Delaunay triangulation} (DT) with a time complexity of $\mathcal{O}(N \log N)$ prior to its clustering procedure, where $N$ is the number of video frames.
Recent methods such as \cite{ma2022graph} achieve state-of-the-art (SOTA) performance using sparse dictionary selection, with significantly higher complexities.
Solving the optimization objective requires two nested iterative loops, each involving costly full matrix inversion, incurring $\mathcal{O}(N^3)$ complexity per iteration.

Keyframe extraction bears a strong resemblance to the \textit{graph sampling} problem in \textit{graph signal processing} (GSP)~\cite{ortega18ieee,cheung18}: 
given a finite graph with edge weights that encode similarities between connected node pairs, choose a node subset to collect samples, so that the reconstruction quality of an assumed smooth (or low-pass) graph signal is optimized.
In particular, a recent fast graph sampling scheme called \gdas{}~\cite{bai2020fast}, which aligns left-ends of \textit{Gershgorin discs} based on the well-known \textit{Gershgorin circle theorem} (GCT)~\cite{varga04} in linear algebra, achieves competitive performance while running in linear-time.
Leveraging \gdas, we propose an efficient keyframe extraction algorithm, \textit{the first in the literature to pose video summarization as a graph sampling problem}, while minimizing a global signal reconstruction error metric.

Specifically, we first construct an $M$-hop \textit{elaborated path graph} (EPG) $\cG^o$ for an $N$-frame video sequence, where $M \ll N$, and the weight $w^o_{i,j}$ of each edge $(i,j)$ connecting two nodes (frames) within $M$ time instants
is determined based on feature distance between the corresponding feature vectors, $\f_i$ and $\f_{j}$.
Then, to reduce sampling complexity, we ``unfold'' the $M$-hop EPG $\cG^o$ to a $1$-hop EPG $\cG$, possibly with self-loops, via one of two graph unfolding procedures with provable performance bounds.
Given a generalized graph Laplacian matrix $\cL$ specifying $\cG$, we show that maximizing the smallest eigenvalue $\lambda_{\min}(\B)$ of a coefficient matrix $\B = \diag{\h} + \mu \cL$, where $\h$ is the binary keyframe selection vector, is equivalent to minimizing a worst-case signal reconstruction error.

Instead, we maximize a lower bound $\lambda_{\min}^-(\B)$ based on GCT by choosing $\h$ via a new fast graph sampling algorithm that iteratively aligns left-ends of Gershgorin discs for all nodes (frames).
If $1$-hop EPG $\cG$ contains self-loops after unfolding, then we first align Gershgorin disc left-ends of $\cL$ via a similarity transform $\S \cL \S^{-1}$, leveraging a recent theorem called \textit{Gershgorin Disc Perfect Alignment} (GDPA) \cite{yang2021signed}. 
Experimental results show that our algorithm achieves comparable or better keyframe selection performance compared to SOTA methods~\cite{mei2014l2,cong2017adaptive,mei2021patcha,ma2022graph}, at a substantially reduced computation complexity.

This paper significantly extends our conference version~\cite{sahami2022fast} to improve video summarization performance:
\begin{enumerate}[label=(\roman*)]
\item \textbf{Graph Construction}: Unlike simple path graphs without self-loops used in~\cite{sahami2022fast}, we initialize an $M$-hop path graph for each transitory video, which is more informative in capturing pairwise similarities between frames.
\item \textbf{Graph Unfolding}: We introduce a new graph sparsification technique, \emph{graph unfolding}, which transforms the original dense $M$-hop EPG into a sparse 1-hop graph (possibly with self-loops) in \textit{linear-time}. 
  We show analytically that this transformation preserves a global regularization bound, enabling efficient processing while maintaining performance guarantees (see Section~\ref{subsec:unfolding}).
\item \textbf{A New Linear-Time Graph Sampling Algorithm:} Unlike \cite{bai2020fast}, which operates on general graphs, our newly developed sampling algorithm is tailored specifically for the unfolded 1-hop path graph and achieves $\mathcal{O}(N)$ runtime with provable guarantees (see Sec.~\ref{subsec:GS-Complexity}). We show that it consistently outperforms the general-graph method in~\cite{bai2020fast} in both efficiency and keyframe selection quality (Table~\ref{tab:ablation1}).
\item \textbf{New Dataset and Extended Experiments}: We conduct extensive summarization experiments on multiple short video datasets, including a newly created dataset specifically for keyframe selection of transitory videos.
\end{enumerate}

The remainder of this paper is organized as follows.
In Section~\ref{sec:relwork}, we review related works.
In Section~\ref{sec:prelim}, we introduce necessary mathematical notations, graph definitions, and GCT.
We describe our graph construction for transitory videos in Section~\ref{sec:graph}.
In Section~\ref{sec:gs_formulation}, we present our graph unfolding procedures and the graph sampling optimization formulation.
In Section~\ref{sec:gs_algorithm}, we detail our graph sampling algorithm applied to the unfolded $1$-hop path graph, considering both scenarios with and without self-loops.
Experimental results are discussed in Section~\ref{sec:results}. 
We conclude the paper in Section~\ref{sec:conclude}.

\section{Related Work}
\label{sec:relwork}

\subsection{Graph Sampling}
\label{subsec:GS}

Subset selection methods in the graph sampling literature select a subset of nodes to collect samples, so that smooth signal reconstruction quality can be maximized.
There exist deterministic and random techniques.
Random approaches~\cite{puy2018random, puy2018structured} select nodes based on assumed probabilistic distributions.
These methods are fast but require more samples to achieve similar signal reconstruction quality.
Deterministic methods~\cite{chen2015_sampling, anis2016efficient, wang2018aoptimal} extend Nyquist sampling to irregular data kernels described by graphs, where frequencies are defined by eigenvalues of a chosen (typically symmetric) graph variation operator, such as the graph Laplacian matrix~\cite{ortega2022introduction}.

Deterministic methods that require first computing eigen-decomposition of a variation operator to determine graph frequencies are computation-intensive and not scalable to large graphs.
To address this, \cite{wang2018aoptimal} employs Neumann series to mitigate eigen-decomposition, but the required large number of matrix multiplications is still expensive.
\cite{sakiyama2019edfree,akie2018eigenfree} circumvents eigen-decomposition by using Chebyshev polynomial approximation, but the method is ad-hoc in nature and lacks a global performance guarantee.
Leveraging GCT, \gdas{}~\cite{bai2020fast} maximizes a lower bound of the smallest eigenvalue $\lambda_{\min}^-(\B)$ of a coefficient matrix $\B = \diag{\h} + \mu \cL$---corresponding to a worst-case signal reconstruction error---by aligning Gershgorin disc left-ends.

In this paper, we redesign \gdas{} specifically for $1$-hop path graphs, creating a lightweight sampling algorithm that requires only node-by-node scalar computation to align disc left-ends.
Using graph unfolding, we then apply this to $M$-hop EPG graphs. Experiments show our new method outperforms previous eigen-decomposition-free methods~\cite{bai2020fast, sakiyama2019edfree}.

\vspace{-0.1in}
\subsection{Video Summarization}
\label{subsec:VS}

Video summarization can be broadly categorized into \textit{dynamic} and \textit{static} approaches which differ fundamentally~\cite{otani2022video}.
Dynamic video summarization~\cite{yalesong2015tvsuma,zhu2021dsnet,liu2022video,li2017general,hsu2023video,zhong2023semantic,zhu2022relational,zhang2024vssnet,yu2024unsupervised,ren2024timechat,zhu2022learning} generates a shorter video by selecting sub-shots based on independently pre-determined per-shot scores in  the knapsack problem setting.
In contrast, static video summarization~\cite{ma2022graph,mei2021patcha,sahami2022fast,dornaika2018instance,KUANAR20131212} selects representative keyframes while accounting for inter-frame correlations, \ie, selecting one frame would reduce the residual values of neighboring correlated frames, making them less likely to be selected next.
Thus, dynamic algorithms are not directly applicable to our keyframe selection benchmarks and fall outside our scope; \blue{we contrast the two evaluation protocols in Sec.~\ref{subsec:protocols}}.

\blue{The dynamic paradigm is today dominated by deep learning.
Supervised models---DPP-LSTM~\cite{zhang2016videob}, VASNet~\cite{fajtl2019summarizing}, MSVA~\cite{ghauri2021supervisedb}, SSPVS~\cite{li2023progressivea}, and CSTA~\cite{son2024cstaa}---regress per-frame importance scores from human annotations, while adversarial models (SUM-GAN~\cite{mahasseni2017unsupervised}, AC-SUM-GAN~\cite{apostolidis2021acsumgan}) learn them without labels; related attention-based frameworks, e.g., Huang~\&~Wang~\cite{huang2019novel}, pursue the same end.
Being shot-based, these are not native keyframe selectors; we nonetheless re-evaluate representative ones under our keyframe-selection protocol as cross-paradigm baselines (Table~\ref{tab:taxonomy}, Sec.~\ref{sec:results}).}

\blue{Among static approaches, we adopt the more practical \textit{unsupervised} setting, forgoing the keyframe annotation that supervised methods require.}
Further, we restrict our study to the single-modality setting, excluding video-to-text~\cite{jia2024queryoriented} and multimodal methods~\cite{xiao2020querybiased}.

\begin{table}[t]
\centering
\caption{\blue{Representative video-summarization methods by family, paradigm, and worst-case time complexity in the number of frames~$N$ and selected keyframes~$C$, where $C$ is fixed or scales as $C\!\propto\!N$ depending on the setting; supervision is marked U (unsupervised) or S (supervised). Our graph-sampling formulation is the only optimization-based keyframe selector with provably linear complexity.}}
\label{tab:taxonomy}
{\resizebox{\columnwidth}{!}{\begin{tabular}{clcl}
\toprule
Family & Method & Paradigm & Complexity \\
\midrule
\multirow{4}{*}{\rotatebox[origin=c]{90}{Cluster.}}
 & DT~\cite{mundur2006keyframebased}         & Static (U)  & $\mathcal{O}(N\log N)$ \\
 & STIMO~\cite{furini2009stimo}              & Static (U)  & $\mathcal{O}(CN)$ \\
 & VSUMM~\cite{deavila2011vsumm}             & Static (U)  & $\mathcal{O}(CN)$ \\
 & VISON~\cite{almeida2012vison}             & Static (U)  & $\mathcal{O}(N+C^2)$ (heuristic) \\
\midrule
\multirow{7}{*}{\rotatebox[origin=c]{90}{Dict.}}
 & SMRS~\cite{elhamifar2012see}              & Static (U)  & $\mathcal{O}(N^3)$ \\
 & SSDS~\cite{wang2017representative}        & Static (U)  & $\mathcal{O}(N^3\log N)$ \\
 & MSR~\cite{mei2015video}                   & Static (U)  & $\mathcal{O}(C^3N)$ \\
 & AGDS~\cite{cong2017adaptive}              & Static (U)  & $\mathcal{O}(C^2N^2)$ \\
 & NSMIS~\cite{dornaika2018instance}         & Static (U)  & $\mathcal{O}(N^3)$ \\
 & SBOMP~\cite{mei2021patcha}                & Static (U)  & $\mathcal{O}(CN^2 + C^3N)$ \\
 & GCSD~\cite{ma2022graph}                   & Static (U)  & $\mathcal{O}(TN^3)^{\sharp}$ \\
\midrule
\multirow{7}{*}{\rotatebox[origin=c]{90}{Deep}}
 & SUM-GAN~\cite{mahasseni2017unsupervised}  & Dynamic (U) & $\Omega(N)^{\dagger}$ \\
& AC-SUM-GAN~\cite{apostolidis2021acsumgan} & Dynamic (U) & $\Omega(N)^{\dagger}$ \\
 & DPP-LSTM~\cite{zhang2016videob}           & Dynamic (S) & $\Omega(N^2)^{\dagger}$ \\
 & VASNet~\cite{fajtl2019summarizing}        & Dynamic (S) & $\Omega(N^2)^{\dagger}$ \\
 & MSVA~\cite{ghauri2021supervisedb}         & Dynamic (S) & $\Omega(N^2)^{\dagger}$ \\
 & SSPVS~\cite{li2023progressivea}           & Dynamic (S) & $\Omega(N^2)^{\dagger}$ \\
 & CSTA~\cite{son2024cstaa}                  & Dynamic (S) & $\Omega(N)^{\dagger}$ \\
\midrule
\multirow{2}{*}{\rotatebox[origin=c]{90}{GSP}}
 & GDAVS~\cite{sahami2022fast}               & Static (U)  & $\mathcal{O}(ND^2)^{\natural}$ \\
 & \textbf{Ours}                             & Static (U)  & $\mathbf{\mathcal{O}(N)}$ \\
\bottomrule
\end{tabular}}
\vspace{2pt}
{\scriptsize\raggedright
$^{\sharp}$~GCSD nests two iterative loops, each enclosing a dense $N\times N$ matrix inversion ($\mathcal{O}(N^3)$); $T=T_{\mathrm{out}}T_{\mathrm{in}}$ is the product of their iteration counts, whose convergence is not characterized in~\cite{ma2022graph}---we therefore leave $T$ explicit rather than assert a bound.
$^{\natural}$~Conference version: $\mathcal{O}(N^2)$ without dynamic programming, reduced to $\mathcal{O}(ND^2)$ with a DP/memoization implementation ($D\ll N$, the maximum recursion depth). $^{\dagger}$~Deep models incur a separate (dataset-dependent) training cost; at inference they must score all $N$ frames, hence at least $\Omega(N)$, and $\Omega(N^2)$ for the attention- and DPP-based methods (full self-attention / a dense $N\times N$ kernel) whereas our algorithm is training-free.
\par}
}
\end{table}

The need for computation efficiency in video summarization has been recognized \cite{lu2017unsupervised, cong2012scalable}.
Previous efforts focusing on low-complexity fall under two categories.
The first includes clustering-based approaches such as Delaunay clustering \cite{mundur2006keyframebased,KUANAR20131212} and hierarchical clustering for scalability \cite{herranz2010framework}.
\cite{deavila2011vsumm} employs $k$-means clustering to select keyframes.
Furini et al. \cite{furini2009stimo} poses summarization as a metric $k$-center problem, a known NP-hard problem.
Given a fixed number of iterations, it operates within a polynomial complexity regime, and only recently an $\mathcal{O}(N\log N)$ implementation is developed \cite{tiwari2020banditpam}.
These methods typically exhibit super-linear computational complexity, escalating significantly when clustering iterations increase.
In contrast, our algorithm has linear-time complexity.

The second category relies on random selection or heuristic-based paradigms~\cite{almeida2012vison, deavila2011vsumm, furini2009stimo}.
Among them, VISON \cite{almeida2012vison} achieves $\mathcal{O}(N)$ complexity.
The algorithm employs a dissimilarity measure to detect similar consecutive frames and groups them based on a threshold, which is not easy to pre-set.

Optimization-based approaches for (unsupervised) static video summarization are studied recently with a focus on the dictionary selection paradigm~\cite{elhamifar2012see,zhao2014quasi,cong2012scalable,mei2014l2,elhamifar2016dissimilaritybased,liu2014diversified,cong2017adaptive,ma2019video,ma2020similarity,mei2021patcha,mei2015video,ma2022graph}.
Notably, \cite{cong2012scalable} and \cite{elhamifar2012see} pose video summarization as a sparse dictionary selection problem, aiming to identify a subset of keyframes for optimal video reconstruction.
\cite{mei2015video} incorporates the $\ell_0$-norm constraint in dictionary learning directly to promote a sparse solution. 
\cite{mei2021patcha} extends the sparse dictionary selection approach to include visual features for frame patches.
This concept is then extended to temporally adjacent frames, forming spatio-temporal block representations.
Subsequently, a greedy approach employing \textit{simultaneous block orthogonal matching pursuit} (SBOMP) is devised.

In \cite{ma2022graph}, graph convolutional dictionary selection with $\ell_{2,p}$ norm ($0 < p \leq 1$) is introduced, marking the first exploration of structured video information represented as a graph within the dictionary selection framework.
In contrast, we introduce graph sampling as a novel paradigm and design a linear-time sampling algorithm for $1$-hop path graphs based on \gdas.

\subsection{\blue{Evaluation Protocols}}
\label{subsec:protocols}

\blue{Keyframe selection (the static paradigm) and shot-based summarization (the dynamic paradigm) are evaluated under two distinct, non-interchangeable protocols.
The static paradigm uses the \textit{one-to-one keyframe-matching} protocol~\cite{deavila2011vsumm,mei2021patcha,ma2022graph}: each selected keyframe is matched to at most one ground-truth keyframe (a bipartite one-to-one matching), and accuracy is the F$_1$ score over the matched set.
The dynamic paradigm uses the \textit{keyshot frame-overlap} protocol~\cite{zhang2016videob}: predicted and ground-truth keyframes are expanded into temporal keyshots (capped at $15\%$ of the video length) and scored by an overlap-based F$_1$, or alternatively by rank correlation (Kendall's $\tau$ / Spearman's $\rho$)~\cite{otani2022video}.
This distinction stems from divergent underlying assumptions:
keyshot frame-overlap presumes a \emph{continuous} per-frame importance signal, whereas one-to-one keyframe matching presumes a \emph{sparse} set of discrete keyframes; a single method can be scored under both only if it produces both output types.
We adopt the one-to-one keyframe-matching protocol throughout, as it is native to the keyframe-selection task we address.}

\section{Preliminaries}
\label{sec:prelim}
\subsection{Graph Definitions}

A positive undirected graph $\cG(\cV,\cE,\W)$ is defined by a set of $N$ nodes $\cV = \{1, \ldots, N\}$, edge set $\cE$, and an \textit{adjacency matrix} $\W \in \mathbb{R}^{N \times N}$.
Edge set $\cE$ specifies edges $(i,j)$, each with weight $w_{i,j}$, and may contain self-loops $(i,i)$ with weight $w_{i,i}$.
Matrix $\W$ contains the weights of edges and self-loops, \ie, $W_{i,j} = w_{i,j} \in \mathbb{R}^+$.
Diagonal \textit{degree matrix} $\D$ has diagonal entries $D_{i,i} = \sum_{j} W_{i,j}, \forall i$. 
A \textit{combinatorial graph Laplacian matrix} $\L$ is defined as $\L \triangleq \D - \W$, which can be proven to be \textit{positive semi-definite} (PSD), \ie, all eigenvalues of $\L$ are non-negative \cite{cheung18}. 
To account for self-loops, the \textit{generalized graph Laplacian matrix} $\cL$ is $\cL \triangleq \D - \W + \diag{\W}$. 

A graph signal $\x \in \mathbb{R}^N$ is smooth with respect to graph $\cG$ if its \textit{graph Laplacian regularizer} (GLR) \cite{pang17} is small:
\begin{align}
\x^{\top} \cL \x = \sum_{(i,j) \in \cE} w_{i,j} (x_i - x_j)^2 + \sum_{i \in \cV} w_{i,i} x_i^2 .
\label{eq:glr}
\end{align}
GLR is commonly used to regularize ill-posed inverse problems such as denoising and dequantization \cite{pang17,liu17}.

\subsection{Gershgorin Circle Theorem (GCT)}
\label{subsec:GCT}

Given a real symmetric matrix $\M$, corresponding to each row $i$ is a \textit{Gershgorin disc} $i$ with center $c_i \triangleq M_{i,i}$ and radius $r_i \triangleq \sum_{j \neq i} |M_{i,j}|$. 
GCT \cite{horn2012matrix} states that each eigenvalue $\lambda$ resides in at least one Gershgorin disc, \ie, $\exists i$ such that 
\begin{align}
c_i - r_i \leq \lambda \leq c_i + r_i .
\label{eq:GCT}
\end{align}

A corollary of \eqref{eq:GCT} is that the smallest Gershgorin disc left-end $\lambda^-_{\min}(\M)$ is a lower bound for the smallest eigenvalue $\lambda_{\min}(\M)$ of $\M$, \ie,
\begin{align}
\lambda^-_{\min}(\M) \triangleq \min_i c_i - r_i \leq \lambda_{\min}(\M) .
\label{eq:GCTc}
\end{align}
A \textit{similarity transform}  $\S \M \S^{-1}$ for an invertible matrix $\S$ has the same set of eigenvalues as $\M$ \cite{horn2012matrix}.
Thus, a GCT lower bound for $\lambda_{\min}(\S \M \S^{-1})$ is also a lower bound for $\lambda_{\min}(\M)$, \ie, for any invertible $\S$, 
\begin{align}
\lambda^-_{\min}(\S \M \S^{-1}) \leq \lambda_{\min}(\M) .
\label{eq:GCT2}
\end{align}
Because the GCT eigenvalue lower bound of the original matrix $\M$ \eqref{eq:GCTc} is often loose, typically, a suitable similarity transform is first performed to obtain a tighter bound \eqref{eq:GCT2}. 
We discuss our choice of transform matrix $\S$ in the sequel. 
 
\section{Video Graph Construction}
\label{sec:graph}

We first describe an efficient method, with linear-time complexity, to construct a sparse path graph $\cG^o$ to represent a transitory video with $N$ frames, which will be used for graph sampling to select keyframes in Sections\;\ref{sec:gs_formulation} and \ref{sec:gs_algorithm}.

\subsection{Graph Connectivity}\label{subsec:GrConn}

A video $\cU$ is a sequence of $N$ consecutive frames $F_i$'s indexed in time by $i$, \ie, $\cU=\left\{F_i\right\}_{i=1}^N$.
To represent $\cU$, we initialize an undirected positive graph $\mathcal{G}^o(\mathcal{V}^o, \mathcal{E}^o, \W^o)$ with node and edge sets, $\cV^o$ and $\cE^o$ respectively, and adjacency matrix $\W^o \in \mathbb{R}^{N \times N}$. 
Each node $i \in \cV^o$ represents a frame $F_i$.
It is connected to a \textit{neighborhood} of (at most) $2M$ nodes $\cN^o_i$, representing frames within temporal distance $M$:

\vspace{-0.05in}
\begin{small}
\begin{equation}
\cN^o_i \triangleq \left\{\max(i-M, 1),\dots,i-1, i+1, \dots, \min(i+M,N) \right\}
\end{equation}
\end{small}\noindent
where $\cN^o_i \subseteq \cV^o$.
We call this graph an $M$-\textit{Elaborate} \textit{Path} \textit{Graph} or $M$-EPG for brevity (see Fig.\;\ref{fig:vsg2}(a)).

The parameter $M$ allows flexibility in specifying graph density.
For transitory videos with non-repetitive scene changes, a small $M$ (\ie, $M \ll N$) can capture the essential temporal inter-frame similarities.
Later in Section~\ref{subsec:ablation} (Table~\ref{tab:ablation1}), we empirically show the efficacy of this graph connectivity.

\subsection{Edge Weight Computation}
\label{subsec:edgecomputation}

Given chosen graph connectivity, we next compute edge weight $w_{i,j}^o$ for each connected pair $(i,j) \in \cE^o$. 
Associated with each frame $F_i$ (node $i$) is a \textit{feature vector} $\f_i \in \mathbb{R}^\ell$, which can be computed leveraging existing technologies \cite{radford2021learning,szegedy2015going}.
Our choice is CLIP (specifically, ViT-B/32)\cite{radford2021learning}.
\blue{We use it as a pretrained encoder, mapping each frame $F_i$ to feature vector $\f_i$ (frames are resized to $224 \times 224$ as input).}
\blue{Learned via image--text contrastive training, these vectors capture high-level semantics, so that similar frames $F_i$ are mapped to features $\f_i$ that are close in feature space, consistent with the smoothness assumption on our graph signal (Sec.~\ref{sec:results}).}

Specifically, to quantify the similarity between a frame pair, we compute an exponential kernel of the Euclidean distance between their respective features~\cite{radford2021learning}, where $\sigma$ controls the sensitivity to distance:
\begin{align}
	w_{i,j}^o = \exp\left(-\frac{\|\f_i - \f_j\|^2_2}{\sigma^2}\right).
  \label{eq:edgeWeight}
\end{align}

Note that alternative similarity metrics can also be considered to compute edge weights $w_{i,j}^o$.

\subsection{Graph Construction Complexity}\label{subsec:gccomplexity}

We show that the complexity of our graph construction method is linear with respect to the number of frames in the video.
\blue{Each frame is encoded independently by a fixed pretrained encoder, so generating feature vector $\f_i$ for frame $F_i$ incurs a cost independent of the number of frames $N$ (though, like any feature extractor, it scales with the frame resolution).}
Consequently, generating frame features for the entire video incurs a computational complexity of $\mathcal{O}(N)$, while graph construction of $M$-EPG has a complexity of $\mathcal{O}(MN\ell)$, where $\ell$ denotes the fixed dimensionality of feature vectors.
Thus, the overall computational complexity of graph construction remains $\mathcal{O}(N\ell)$ since $M \ll N$.

\begin{figure}
    \centering
    \begin{minipage}{0.5\columnwidth}
        \subfloat[a][$3$-EPG]{\includegraphics[width=\linewidth]{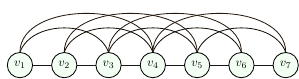}}
    \end{minipage}\begin{minipage}{0.5\columnwidth}
        \subfloat[b][$\beta=2$]{\includegraphics[width=\linewidth]{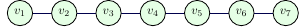}}\hfill
        \subfloat[c][$\beta=0$]{\includegraphics[width=\linewidth]{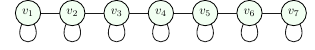}}
    \end{minipage}
	\vspace{-0.15in}
	\caption{Illustration of the $M$-EPG graph (a) and its unfolded versions: (b) Graph Unfolding 1 corresponding to $\beta=2$, and (c) Graph Unfolding 2 corresponding to $\beta=0$}
    \label{fig:vsg2}
\end{figure}

\section{Graph Sampling: Formulation}
\label{sec:gs_formulation}

We first develop methodologies to ``unfold'' an $M$-EPG to a sparser $1$-EPG for more efficient sampling in Section\;\ref{subsec:unfolding}.
We formulate a sampling objective for a $1$-EPG in Section\;\ref{subsec:sampling_obj}.
Finally, we develop sampling algorithms for $1$-EPGs without and with self-loops in Section\;\ref{subsec:sampling_no_selfloop} and \ref{subsec:sampling_selfLoops}, respectively.

\subsection{Graph Unfolding}
\label{subsec:unfolding}

We first discuss how we \textit{unfold} an original $M$-EPG $\cG^o$, specified by a graph Laplacian matrix $\L^o$, to a $1$-EPG $\cG$ [called \textit{simple path graph} (SPG)], possibly with self-loops, specified by a tri-diagonal generalized graph Laplacian matrix $\cL$.
SPGs are more amenable to fast graph sampling.
Specifically, we seek a Laplacian $\cL$ for an SPG $\cG$ such that
\begin{align}
\x^\top \L^o \x \leq \x^\top \cL \x, ~~~\forall \x \in \mathbb{R}^N .
\label{eq:GLR_bound}
\end{align}
where $\x^\top \L^o \x$ is a GLR \eqref{eq:glr} \cite{pang17} that quantifies the variation of signal $\x$ across a graph kernel specified by $\L^o$.

GLR has been used as a signal prior in graph signal restoration problems such as denoising, and dequantization \cite{pang17,liu17}.
Replacing $\x^\top \L^o \x$ with $\x^\top \cL \x$ means that, in a minimization problem, we are minimizing an \textit{upper bound} of GLR $\x^\top \L^o \x$.

\eqref{eq:GLR_bound} implies that $\cL - \L^o \succeq 0$ is PSD.
The next theorem ensures the PSDness of $\cL - \L^o$ when converting an edge $(i,j) \in \cE^o$ of positive weight $w^o_{i,j} > 0$ in the original $M$-EPG $\cG^o$ to possible edges $(i,k)$ and $(k,j)$ and self-loops at nodes $i$ and $j$ and intermediate node $k$ in the unfolded SPG $\cG$.

\begin{theorem}
Given a graph $\cG^o$ specified by generalized graph Laplacian $\L^o$, to replace an edge $(i,j)$ of weight $w^o_{i,j} > 0$ connecting nodes $i$ and $j$ in $\cG^o$, a procedure that adds edges $(i,k)$ and $(k,j)$ to/from intermediate node $k$ in graph $\cG$, each of weight $\beta w^o_{i,j}$, adds self-loops at nodes $i$ and $j$, each of weight $(2-\beta) w^o_{i,j}$, and adds self-loop at node $k$ of weight $(\beta^2 - 2\beta) w^o_{i,j}$, for $\beta \in \mathbb{R}$, results in generalized graph Laplacian $\cL$ for modified graph $\cG$ such that $\cL - \L^o$ is PSD.
\label{thm:graph_transform}
\end{theorem}
See Fig.\;\ref{fig:theorem1} for an illustration of an edge $(i,j)$ in original graph $\cG^o$ replaced by edges $(i,j-1)$ and $(j-1,j)$ and self-loops at nodes $i$, $j-1$ and $j$, each with different weights.

\begin{proof}
GLR $\x^\top \cL \x$ for generalized Laplacian $\cL$ of graph $\cG$ with edge weights $w_{i,j}$'s and self-loop weights $u_i$'s, can be written as a weighted sum of connected sample pair difference squares $(x_i - x_j)^2$ and sample energies $x_i^2$ \eqref{eq:glr}:
\begin{align}
\x^\top \cL \x = \sum_{(i,j) \in \cE} w_{i,j} (x_i - x_j)^2 + \sum_{i \in \cV} u_i x_i^2 .
\end{align}

Considering only edge $(i,j) \in \cE^o$ in graph $\cG^o$, its contribution to GLR $\x^\top \L^o \x$ is
\begin{align}
w^o_{i,j} (x_i - x_j)^2 = w^o_{i,j} (x_i^2 - 2 x_i x_j + x_j^2) .
\end{align}

Algebraically, starting from $0 \leq y^2, ~\forall y \in \mathbb{R}$, we write
\begin{align}
0 \leq (a - \beta b + c)^2 &= a^2 - 2 \beta ab + 2ac + \beta^2 b^2 - 2 \beta bc + c^2
\nonumber \\
-2 ac & \leq \beta (a-b)^2 + \beta (b-c)^2 
\nonumber \\
& ~~ + (1-\beta) a^2 + (\beta^2 - 2\beta) b^2 + (1-\beta) c^2 .
\nonumber 
\end{align}
Letting $a = x_i$, $c = x_j$ and $b = x_k$ for some intermediate node $k$, we can now upper-bound $(x_i - x_j)^2$ as
\begin{align}
x_i^2 - 2 x_i x_j + x_j^2 &\leq \beta (x_i - x_k)^2 + \beta (x_k - x_j)^2 
\nonumber \\
& ~~ + (2-\beta) x_i^2 + (\beta^2 - 2\beta) x_k^2 + (2-\beta) x_j^2
\nonumber 
\end{align}

Thus, replacing edge $(i,j)$ of positive weight $w^o_{i,j}$ in $\cG^o$ with edges $(i,k)$ and $(k,j)$ to/from intermediate node $k$, each with weight $\beta w^o_{i,j}$, and self-loops at nodes $i$, $k$ and $j$ with weights $(2-\beta) w^o_{i,j}$, $(\beta^2 - 2\beta) w^o_{i,j}$ and $(2-\beta) w^o_{i,j}$ in $\cG$, means that $\x^\top \L^o \x \leq \x^\top \cL \x, \forall \x$, or $\cL - \L^o \succeq 0$.
\end{proof}

\begin{figure}
\begin{center}
	\includegraphics[width=0.9\linewidth]{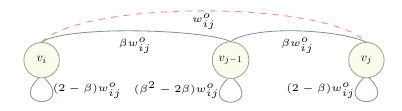}
\end{center}
\vspace{-0.2in}
\caption{Graphical representation of Theorem~\ref{thm:graph_transform}: Replacing a (dashed red) edge $(i,j)$ with weight $w^o_{i,j}$ in the original graph $\cG^o$ with (solid green) edges $(i,j-1)$ and $(j-1,j)$ and self-loops at nodes $i$, $j-1$ and $j$ of different weights, where $j-1$ is an intermediate node. }
\label{fig:theorem1}
\end{figure}

Given Theorem\;\ref{thm:graph_transform}, we deduce two practical corollaries.
The first corollary follows from Theorem\;\ref{thm:graph_transform} for $\beta = 2$.
\begin{corollary}
Given a graph $\cG^o$ specified by generalized graph Laplacian $\L^o$, to replace an edge $(i,j)$ connecting nodes $i$ and $j$ in $\cG^o$ of weight $w^o_{i,j} > 0$, a procedure that adds edges $(i,k)$ and $(k,j)$ to/from intermediate node $k$ in graph $\cG$, each of weight $2 w^o_{i,j}$, results in generalized graph Laplacian $\cL$ such that $\cL - \L^o$ is PSD.
\label{corollary:newEdges}
\end{corollary}
The second corollary follows from Theorem\;\ref{thm:graph_transform} for $\beta = 0$.
\begin{corollary}
Given a graph $\cG^o$ specified by generalized graph Laplacian $\L^o$, to replace an edge $(i,j)$ connecting nodes $i$ and $j$ in $\cG^o$ with weight $w^o_{i,j} > 0$, a procedure that adds self-loops at nodes $i$ and $j$ in graph $\cG$, each of weight $2 w^o_{i,j}$, results in generalized graph Laplacian $\cL$ for modified graph $\cG$ such that $\cL - \L^o$ is PSD.
\label{corollary:self-loops}
\end{corollary}

The two corollaries represent two extreme instantiations of Theorem\;\ref{thm:graph_transform}: the replacement procedure requires only adding edges/self-loops, respectively.
They also lead to two graph unfolding procedures of different complexities to transform an $M$-hop EPG $\cG^o$ to an SPG $\cG$:
\vspace{0.05in}
\\
\textbf{Graph Unfolding 1}: 
to replace each positive edge $(i,j)$ in $\cG^o$ with weight $w^o_{i,j}$ further than one hop, add edges $(i,j-1)$ and $(j-1,j)$ in $\cG$ with weight $w_{i,j-1} = w_{j-1,j} = 2 w^o_{i,j}$.
\\
\textbf{Graph Unfolding 2}: 
to replace each positive edge $(i,j)$ in $\cG^o$ with weight $w^o_{i,j}$ further than one hop, add self-loops at nodes $i$ and $j$ in $\cG$ with weights $u_i = u_j = 2 w^o_{i,j}$.
\vspace{0.05in}

Specifically, in the first unfolding procedure, each unfolding of an edge $(i,j)$ in $\cG^o$ requires $j-1$ replacement steps towards an SPG, thus a higher procedural complexity.

\subsection{Defining Graph Sampling Objective}
\label{subsec:sampling_obj}

Given an unfolded SPG $\cG$ specified by generalized Laplacian $\cL$ (see Fig.\;\ref{fig:vsg2}(b) for an illustration), we now aim to select $C$ representative sample nodes (frames) from $\cG$, where $C<N$.
To define a sampling objective, we first derive the worst-case reconstruction error of signal $\x \in \mathbb{R}^N$ given $C$ chosen samples.

Using \textit{sampling matrix} $\H \in \left\{0,1\right\}^{C \times N}$ defined as
\begin{align}
H_{i,j} = \left\{ \begin{array}{ll} 
1 & \mbox{if node $j$ is the $i$-th sample} \\
0 & \mbox{o.w.}
\end{array} \right. ,
\end{align}
one can choose $C$ samples from signal $\x \in \mathbb{R}^N$ as $\y = \H \x$.
To recover original signal $\x$ given observed samples $\y \in \mathbb{R}^C$, we employ GLR \eqref{eq:glr} as a signal prior \cite{pang17}, resulting in the following restoration problem:
\begin{align}
\min_{\x} \|\y - \H \x \|^2_2 + \mu \, \x^{\top} \cL \x
\label{eq:reconObj}
\end{align}
where $\mu > 0$ is a weight parameter that balances between the data fidelity and prior terms.
Given \eqref{eq:reconObj} is quadratic, convex, and differentiable, its solution $\x^*$ can be obtained by solving a system of linear equations:
\begin{align}
(\H^\top \H + \mu \cL) \x^* = \H^\top \y .
\end{align} 
The coefficient matrix $\B \triangleq \H^{\top} \H + \mu \cL$ is provably positive definite (PD) and thus invertible\footnote{\blue{This is a generalization of the positive-definiteness result in \cite{bai2020fast}, where a combinatorial Laplacian for a positive graph without self-loops is considered. 
Here, we consider a generalized Laplacian for a positive connected graph possibly with self-loops, whose first eigenvector is strictly positive \cite{yang2021signed}.}}, given $\cL$ is a generalized graph Laplacian for a positive connected graph.
\begin{lemma}
Coefficient matrix $\B = \H^\top \H + \mu \cL$, where $\H \in \{0,1\}^{C \times N}$ is a sampling matrix, $\cL$ is a \blue{generalized  Laplacian for} a positive connected graph possibly with self-loops, and $\mu > 0$, is PD for $C \geq 1$.
\end{lemma}
\begin{proof}
Note first that both $\H^\top \H$ and $\cL$ are PSD (generalized Laplacian $\cL$ for a positive graph $\cG$, with or without self-loops, is provably PSD via GCT \cite{cheung18}).
Thus, $\H^\top \H + \mu \cL$ is PD iff there does not exist a vector $\v$ such that $\v^\top \H^\top \H \v$ and $\v^\top \cL \v$ are simultaneously $0$.
By Lemma 1 in \cite{yang2021signed}, the lone first eigenvector $\v_1$ corresponding to the smallest eigenvalue $\lambda_1 \geq 0$ of a generalized Laplacian $\cL$ for a positive connected graph $\cG$ is \textit{strictly positive} (proven via the Perron-Frobenius Theorem), \ie, $v_{1,i} > 0, \forall i$.
Thus, $\v_1^\top \H^\top \H \v_1 = \sum_{i \in \cS} v_{1,i}^2 > 0$, where $\cS$ is the non-empty index set for the chosen $C \geq 1$ sample nodes.
So while $\v_1^\top \cL \v_1 = 0$ if $\lambda_1 = 0$, $\v_1^\top \H^\top \H \v_1 > 0$, and hence there are no vectors $\v$ such that $\v^\top \H^\top \H \v = \v^\top \cL \v = 0$.
Therefore, $\B = \H^\top \H + \mu \cL$ is PD for $\mu > 0$.
\end{proof}
Note that $\H^{\top} \H$ is an $N \times N$ diagonal matrix, whose diagonal entries corresponding to $C$ selected nodes are $1$, and $0$ otherwise.
For convenience, we define $\h \triangleq \{0,1\}^N$ as a 0-1 vector of length $N$, and $\H^\top \H = \diag{\h}$.

One can show that maximizing the smallest eigenvalue $\lambda_{\min}(\B)$ of coefficient matrix $\B$---known as the \textit{E-optimality criterion} in system design~\cite{ehrenfeld1955efficiency}---is equivalent to minimizing a worst-case reconstruction error (Proposition 1 in \cite{bai2020fast}).
Given a sampling budget $C$, our sampling objective is thus to maximize $\lambda_{\min}(\B)$ using $\h$, \ie, 
\begin{align}
\max_{\h} ~ \lambda_{\min}(\diag{\h} + \mu \cL),
~~~ \mbox{s.t.}~ \|\h \|_1 \leq C .
\label{eq:sample_obj0}
\end{align}

For ease of computation, instead of $\lambda_{\min}(\B)$ we maximize instead the \textit{GCT lower bound} $\lambda^-_{\min}(\S \B \S^{-1})$ of the similarity transform $\S \B \S^{-1}$ of matrix $\B$ (with same set of eigenvalues as $\B$), where $\lambda^-_{\min}(\S \B \S^{-1}) \leq \lambda_{\min}(\B)$, \ie,
\begin{align}
\max_{\h, \S} ~ \lambda^-_{\min} \left( \S (\diag{\h} + \mu \cL)  \S^{-1} \right), 
~~\mbox{s.t.}~\|\h\|_1 \leq C 
\label{eq:sample_obj1}
\end{align}
where $\S$ is an invertible matrix.
In essence, \eqref{eq:sample_obj1} seeks to maximize the smallest Gershgorin disc left-end of matrix $\S \B \S^{-1}$ using $\h$ and $\S$, under a sampling budget $C$.
For simplicity, we consider only positive diagonal matrices for $\S$, \ie, $\S = \diag{\s}$ and $s_i > 0, \forall i$.

Note that our objective \eqref{eq:sample_obj1} differs from \cite{bai2020fast} in two respects.
First, $\cL$ in \eqref{eq:sample_obj1} is a generalized graph Laplacian for a positive graph possibly \textit{with} self-loops, while \cite{bai2020fast} considers narrowly a positive graph \textit{without} self-loops.
Second, $\cL$ corresponds to an SPG $\cG$ and thus is tri-diagonal, while the combinatorial Laplacian in \cite{bai2020fast} corresponds to a more general graph with unstructured node-to-node connectivity.

We remark briefly on the interpretation of the target signal and on alternative smoothness priors.
\blue{The graph is built solely from feature similarity, with edge weights given by \eqref{eq:edgeWeight}---so that frames with similar content are strongly connected.
	Our formulation relies only on the target signal being \textit{smooth} over this graph---a mild assumption, since adjacent frames with similar content vary slowly.}
An alternative is to assume $\x$ is $\omega$-\textit{bandlimited} and pursue \textit{A-optimal sampling} (minimize average reconstruction error), which requires costly matrix inverse and eigen-decomposition~\cite{wang2018aoptimal,wang2023fast}.
In contrast, the GLR prior yields a sparse, symmetric positive-definite normal matrix $\B$, admits fast linear-time solvers for signal reconstruction, and leads to an E-optimal sampling objective (maximize the smallest eigenvalue) that is easier to optimize.
 
\section{GRAPH Sampling: algorithm}
\label{sec:gs_algorithm}

We describe a linear-time algorithm for the formulated graph sampling problem \eqref{eq:sample_obj1}, focusing first on the case of an SPG without self-loops in Section\;\ref{subsec:sampling_no_selfloop}.
We then extend the algorithm to the case of an SPG with self-loops in Section\;\ref{subsec:sampling_selfLoops}.
We study fast sampling exclusively for SPGs, because an $M$-EPG modeling a video can be unfolded into an SPG with or without self-loops via the two graph unfolding procedures discussed in Section\;\ref{subsec:unfolding}.

\subsection{Sampling for Simple Path Graph without Self-loops}
\label{subsec:sampling_no_selfloop}

To design an intuitive algorithm, we first swap the roles of the objective and constraint in \eqref{eq:sample_obj1} and solve instead its corresponding \textit{dual problem} given threshold $0 < T <1$ \cite{bai2020fast}:
\begin{align}
\min_{\h, \S} \|\h\|_1
~~\mbox{s.t.}~~ \lambda^-_{\min} \left( \S (\diag{\h} + \mu \cL) \S^{-1} \right) \geq T .
\label{eq:sample_obj2}
\end{align}
In words, \eqref{eq:sample_obj2} minimizes the number of samples needed to move all Gershgorin disc left-ends of similarity-transformed matrix $\S \B \S^{-1}$ to at least $T$.
It was shown in~\cite{bai2020fast} that $T$ is inversely proportional to objective $\|\h^*\|_1$ of the optimal solution $\h^*$ to \eqref{eq:sample_obj2}.
Thus, to find the dual optimal solution $\h^*$ to \eqref{eq:sample_obj2} such that $\|\h^*\|_1 = C$---and hence the optimal solution $\h^*$ also to the primal problem \eqref{eq:sample_obj1} \cite{bai2020fast}---one can perform binary search to find an appropriate $T$.
We describe a fast algorithm to approximately solve \eqref{eq:sample_obj2} for a given $T$ next.

We first state a lemma\footnote{Our earlier conference version \cite{sahami2022fast} states a lemma relating the smallest eigenvalues of Laplacians $\cL'$ and $\cL$, while here Lemma\;\ref{lemma:reduced_graph} relates the smallest eigenvalues of $\diag{\h} + \mu \cL'$ and $\diag{\h} + \mu \cL$. 
} that allows us to optimize \eqref{eq:sample_obj2} for a \textit{reduced graph} $\cG'$ containing the same node set $\cV$ as the original positive graph $\cG$, and a reduced edge subset $\cE' \subset \cE$.

\begin{lemma}
Denote by $\cG'(\cV,\cE',\W')$ a reduced graph from positive graph $\cG(\cV,\cE,\W)$, where edges $(i,j) \in \cE \setminus \cE'$ were selectively removed.
Denote by $\cL'$ and $\cL$ the generalized Laplacians for graphs $\cG'$ and $\cG$, respectively.
Then,
\begin{align}
\lambda_{\min}\left(\diag{\h} + \mu \cL'\right) \leq \lambda_{\min} (\diag{\h} + \mu\cL).
\end{align}
\label{lemma:reduced_graph}
\end{lemma}
\begin{proof}
Since $\x^\top \cL \x$ is a weighted sum of sample difference squares and sample energies \eqref{eq:glr}, for any $\x \in \mathbb{R}^N$,
\begin{align}
\x^\top (\cL - \cL') \x = \sum_{(i,j) \in \cE \setminus \cE'} w_{i,j} (x_i - x_j)^2 \stackrel{(a)}{\geq} 0
\end{align}
where $(a)$ is true since $w_{i,j} \geq 0$ for positive graph $\cG$.
Thus,
\begin{align}
\x^\top (\diag{\h} + \mu \cL') \x \leq \x^\top (\diag{\h} + \mu \cL) \x, ~~~\forall \x .
\end{align}
Since this inequality includes the first (unit-norm) eigenvector $\v_1$ of $\diag{\h} + \mu \cL$ corresponding to the smallest eigenvalue,
\begin{align}
\v_1^\top (\diag{\h} + \mu \cL') \v_1 \leq \lambda_{\min} (\diag{\h} + \mu \cL) .
\end{align}
Since the left side is at least $\lambda_{\min}(\text{diag}(\h) + \mu \cL')$, the lemma is proven.
\end{proof}

Lemma\;\ref{lemma:reduced_graph} applies to any reduced graph $\cG'$ with removed edges $(i,j) \in \cE \setminus \cE'$. 
Particularly useful in our algorithm development is when $\cG'$ is a \textit{disconnected graph} $\cG'(\cV_1' \cup \cV_2',\cE_1' \cup \cE_2', \W')$ for $\cV_1' \cap \cV_2' = \emptyset$ and $\cV_1' \cup \cV_2' = \cV'$, where edge set $\cE'_1$ ($\cE'_2$) connects only nodes in $\cV'_1$ ($\cV'_2$).
In other words, edges $(i,j)$ where $i \in \cV'_1$ and $j \in \cV'_2$ are all removed in $\cG'$.
With appropriate node reordering, the adjacency matrix for $\cG'$ is block-diagonal, \ie, $\W' = \diag{\W'_1, \W'_2}$, where $\W'_1$ and $\W'_2$ are the adjacency matrices for the two sub-graphs disconnected from each other. 
We state a useful corollary.
\begin{corollary}
Suppose $\cG'(\cV_1' \cup \cV_2',\cE_1' \cup \cE_2', \W')$ is a reduced disconnected graph from original graph $\cG(\cV,\cE,\W)$, where $\W' = \diag{\W'_1, \W'_2}$.
Then,
\begin{align}
\min_{i \in \{1,2\}} \lambda_{\min} (\diag{\h_i} + \mu \cL_i') \leq \lambda_{\min} (\diag{\h} + \mu \cL),
\end{align}
where $\h_i$ is the sub-vector of $\h$ corresponding to nodes $\cV_i'$, $\cL_i'$ is the generalized Laplacian for sub-graph $i$, and $\cL$ is the generalized Laplacian for graph $\cG$.
\label{colloary:divideGraph}
\end{corollary}
\begin{proof}
By Lemma\;\ref{lemma:reduced_graph}, we know that for Laplacian $\cL'$ of reduced disconnected graph $\cG'$,
\begin{align}
\lambda_{\min} (\diag{\h} + \mu \cL') \leq \lambda_{\min} (\diag{\h} + \mu \cL) .
\end{align}
From linear algebra, we also know that for a block-diagonal matrix $\diag{\h} + \mu \cL'$ with sub-block matrices $\diag{\h_1} + \mu \cL'_1$ and $\diag{\h_2} + \mu \cL'_2$,
\begin{align}
\min_{i \in \{1,2\}} \lambda_{\min} (\diag{\h_i} + \mu \cL_i') = \lambda_{\min} (\diag{\h} + \mu \cL') .
\end{align}
Thus, the corollary is proven.
\end{proof}

Given Corollary\;\ref{colloary:divideGraph}, the main idea to solve \eqref{eq:sample_obj2} is \textit{divide-and-conquer}: to maintain threshold $T$ of lower bound $\lambda_{\min}^-(\diag{\h} + \mu \cL)$,  
identify one sample node $k$ such that left-ends of Gershgorin discs corresponding to the first $d$ nodes can be moved beyond $T$.
We then recursively solve the same sampling problem for sub-graph $\cG'$ with subset of nodes $\{d+1, \ldots N\}$.

Specifically, given generalized Laplacian $\cL$ for SPG $\cG$, we solve
\begin{align}
\max_{k, \S^d}~ d, ~~
\mbox{s.t.} ~~\lambda^-_{\min}(\S^d (\diag{\e_k} + \mu \cL^d) (\S^d)^{-1}) \geq T
\label{eq:sample_obj3}
\end{align}
where $\e_k \in \{0,1\}^d$ is a length-$d$ \textit{canonical vector} with only one non-zero entry $e_{k,k}=1$ for $k \in \{1, \ldots, d\}$, $\cL^{d}$ is the generalized Laplacian for the sub-graph containing only the first $d$ nodes, and $\S^d$ is a diagonal $d \times d$ matrix.
In words, \eqref{eq:sample_obj3} seeks one single sample node $k$ and similarity transform matrix $\S^d$, so that Gershgorin disc left-ends of $d \times d$ matrix, $\S^d (\text{diag}(\e_k) + \mu \cL^d)(\S^d)^{-1}$, move beyond $T$.

To solve \eqref{eq:sample_obj3} efficiently, we first assume that SPG $\cG$ has no self-loops.
The implication is that all Gershgorin discs $i$ of generalized Laplacian $\cL$ have left-ends $c_i - r_i$ at $0$, \ie,
\begin{align}
\cL_{i,i} - \sum_{j \neq i} |\cL_{i,j}| =  \sum_{j \neq i} W_{i,j} - \sum_{j \neq i} W_{i,j} = 0, ~~\forall i .
\end{align}
This means that a sample node $k$ and transform matrix $\S^d$ must move all $d$ disc left-ends from $0$ to $T$.
\blue{Sampling node $i$ (setting $h_i = 1$) shifts the center of disc $i$ rightward by $1$, while scaling by $\S^d$ rescales the disc radii; Fig.\;\ref{fig:gda} illustrates these two operations.}

\begin{figure}[t]
	\centering
	\includegraphics[width=0.99\linewidth]{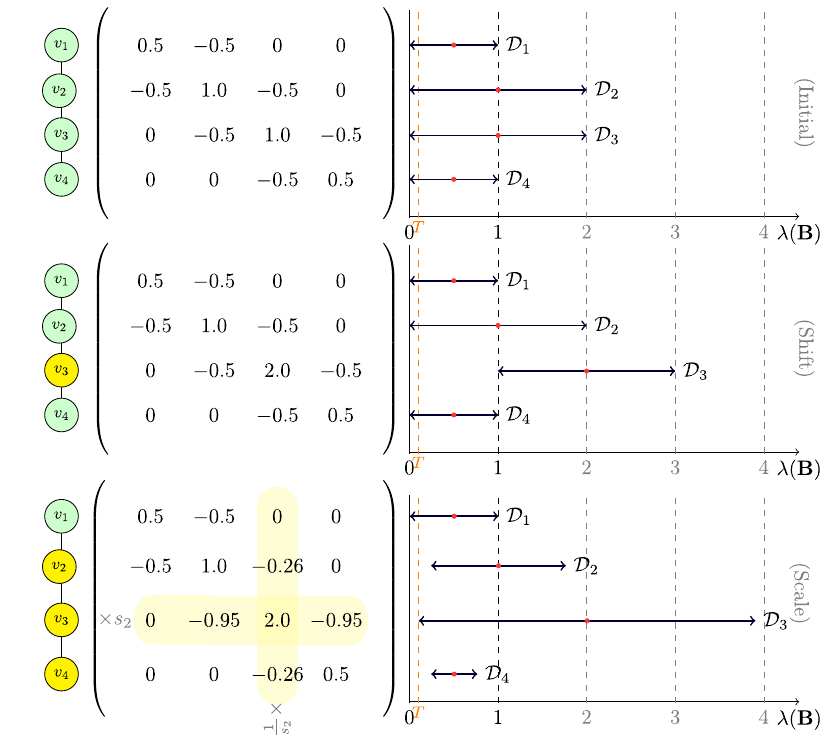}
 \vspace{-0.2in}
	\caption{Illustration of Shifting and Scaling operations in \textit{Gershgorin Disk Alignment}.
		\textcolor{gray}{(Initial)}: The example shows a simple path graph with $4$ nodes denoted by $v_1,v_2,\dots,v_4$, connected by edges with weights $w_{i,i+1}=0.5$, $\forall i=1,2,\dots,3$.
		The matrix represented here is $\mathbf{B}=\diag{\mathbf{h}} + \mu\mathbf{L}$, where $\mathbf{L}$ denotes the Laplacian matrix.
		Initially, all disc left-ends are located at $0$.
		\textcolor{gray}{(Shift)}: Upon sampling node $v_3$, its corresponding disc, $\mathcal{D}_3$, is shifted right by $1$.
		\textcolor{gray}{(Scale)}: The application of scalar $s_3=1.9$ causes $\mathcal{D}_3$'s left-end to align at $T=0.1$, while the reciprocal scalars $1/s_3$ decrease the radii of $\mathcal{D}_2$ and $\mathcal{D}_4$ subsequently.
	}
	\label{fig:gda}
\end{figure}

\subsubsection{Upstream Sampling Algorithm}

We derive an \textit{upstream} procedure to select one sample $k$ and diagonal matrix $\S^d = \diag{\s_d}$ for problem \eqref{eq:sample_obj3}.
\blue{In a nutshell, starting from node $1$, we find the \textit{furthest} sample $k$ that would ``cover''\footnote{By cover, we mean sample $k$ induces changes in Gershgorin disc left-ends of nodes in coverage to be greater than or equal to $T$.} all left-side nodes $\{1, \ldots, k-1\}$ in the SPG.
Having selected sample $k$, we then determine the corresponding right-side coverage $\{k+1, \ldots, d\}$ via a \textit{downstream} procedure to be detailed in Section\;\ref{subsubsec:computing_scalars}.
Unlike the greedy set-cover procedure that \gdas{}~\cite{bai2020fast} employs over pre-computed coverage sets, this upstream sweep selects a sample $k$ directly in a single traversal.}

Starting from node $1$, setting $s_1 = 1$, the minimum scalar $s_2^l > 1$ needed to move left-end of disc $1$ beyond $T$ is
\begin{align}
\cL_{1,1} - s_2^{-1} W_{1,2} &\geq T
\nonumber \\
s_2^l &\triangleq \frac{W_{1,2}}{\cL_{1,1}-T} .
\end{align}
$s_2^l$ is a ``minimum'' value in the sense that any smaller scalar $s_2 < s_2^l$ would not move disc $1$ left-end beyond $T$.
\blue{In words, $s^l_2$ is the minimum \textit{demand} node $1$ is requesting.}

If node $2$ is sampled, then the maximum scalar $s_2$ while ensuring disc $2$ left-end moves beyond $T$ is $s_2^u$:
\begin{align}
\cL_{2,2} + 1 - s_2 (W_{2,1} + W_{2,3}) &\geq T
\nonumber \\
s^u_2 &\triangleq \frac{\cL_{2,2}+1 - T}{W_{2,1} + W_{2,3}} .
\end{align}
\blue{In words, $s^u_2$ is the maximum \textit{supply} node $2$ can provide if sampled.
If $s^u_2 < s^l_2$, then sampling node $2$ does not generate sufficient supply to meet the minimum demand requested by node $1$. 
Hence, sampling node $2$ cannot cover node $1$, and node $1$ must be sampled instead. 
In this case, sample node is $k=1$, and maximal $d$ to determine right-side coverage $\{2,\ldots,d\}$ can be obtained via the to-be-discussed downstream procedure.}

If $s^u_2 \geq s^l_2$, then sampling node $2$ can cover node $1$, and hence $k > 1$. 
But does node $2$ need to be sampled? 
If node $2$ is also not sampled, then using $s_2^l$ (to ensure node $1$ is covered), the minimum scalar (demand) $s_3^l \geq 1$ required to move disc $2$ left-end beyond $T$ is 
\begin{align}
\cL_{2,2} - s_2^l (W_{2,1} + s_3^{-1} W_{2,3}) &\geq T 
\nonumber \\
s_3^l &\triangleq \frac{s_2^l W_{2,3}}{\cL_{2,2} - s^l_2 W_{2,1} - T} .
\end{align}

If node $3$ is a sample node, then the maximum scalar (supply) $s_3^u$ while moving disc $3$ left-end beyond $T$ is
\begin{align}
\cL_{3,3} + 1 - s_3((s_2^l)^{-1} W_{3,2} + W_{3,4} ) &\geq T   
\nonumber \\
s_3^u \triangleq \frac{\cL_{3,3} + 1 - T}{(s_2^l)^{-1} W_{3,2} + W_{3,4}} .
\end{align}
\blue{If $s_3^u < s_3^l$, then the maximum supply provided by sampling node $3$ is not sufficient to satisfy the minimum demand required by node $2$.}
Thus, node $2$ must be sampled.
On the other hand, if $s_3^u \geq s_3^l$, then sampling node $3$ can cover nodes $2$ and $1$.
\blue{Fig.\;\ref{fig:updownstream} visualizes this upstream procedure, tracking demand $s^l_i$ and supply $s^u_i$ along the path until $s^u_i < s^l_i$ to determine sample $k$.}

We can now generalize the previous analysis to an algorithm to determine a single sample node $k$ as follows.

\begin{enumerate}
\item Initialize $s_1^l \leftarrow 1$, $W_{1,0} \leftarrow 0$, and $i \leftarrow 1$.
\item Compute \textit{minimum scalar} $s_{i+1}^l$ required to cover upstream nodes $i, \ldots, 1$:
\begin{align}
s_{i+1}^l \triangleq \frac{s_i^l W_{i,i+1}}{\cL_{i,i} - s_i^l W_{i,i-1} - T}  .
\label{eq:sl}
\end{align}
\item Compute \textit{maximum scalar} $s_{i+1}^u$ if node $i+1$ is sampled.
\begin{align}
s_{i+1}^u \triangleq \frac{\cL_{i+1,i+1}+1-T}{(s_i^l)^{-1} W_{i+1,i} + W_{i+1,i+2}}  .
\label{eq:su}
\end{align}
If $s_{i+1}^u < s_{i+1}^l$ or $s_{i+1}^l<1$, then $k \leftarrow i$ and $s_k^l \leftarrow s_k^u$. Exit.
\item Given $s_{i+1}^u \geq s_{i+1}^l$ and $s_{i+1}^l \geq 1$, increment $i$ and goto step 2.
\end{enumerate}
Given sample $k$, compute $d$ via procedure in Section\;\ref{subsubsec:computing_scalars}.

\subsubsection{Downstream Coverage Computation}
\label{subsubsec:computing_scalars}

We reuse the \textit{downstream} procedure in \cite{bai2020fast} to determine the right-side coverage $\{k+1, \ldots, d\}$ of a chosen sample node $k$.
\blue{Specifically, because $\cG$ is an SPG, the procedure becomes a simple traversal along the path graph and does not require the breadth-first neighborhood search in \cite{bai2020fast}.}

Sampling $k$ means $h_k = 1$, and thus diagonal entry $B_{k,k}$---center of disc $k$ corresponding to node $k$---is increased by $1$.
One can now \textit{expand} the radius of disc $k$ by factor $s_k^u > 1$ so that left-end of disc $k$ is at $T$ \textit{exactly}, \ie, 
\begin{align}
\cL_{k,k} + 1 - s_k (W_{k,k-1} + W_{k,k+1}) &\geq T
\nonumber \\
s_k^u \triangleq \frac{\cL_{k,k}+1-T}{W_{k,k-1}+W_{k,k+1}} .
\end{align}
$s_k^u$ is a ``maximum'' in the sense that any larger scalar $s_k > s_k^u$ would reduce disc $k$ left-end to below $T$.

Setting scalar $s_k^u > 1$ also means reducing the $(k+1,k)$-entry of matrix $\S \B \S^{-1}$ to $(s_k^u)^{-1} W_{k+1,k}$, and thus reducing disc $k+1$ radius.
For disc $k+1$ left-end to move past $T$, the following inequality must be satisfied:
\begin{align}
\cL_{k+1,k+1} - s_{k+1} ((s_k^u)^{-1} W_{k+1,k} + W_{k+1,k+2}) &\geq T   
\nonumber \\
s_{k+1}^u \triangleq \frac{\cL_{k+1,k+1}-T}{(s_k^u)^{-1}W_{k+1,k} + W_{k+1,k+2}} .
\label{eq:sd}
\end{align}

If $s_{k+1}^u < 1$, then sample $k$ cannot cover node $k+1$, and $d = k$.
Otherwise, sample $k$ covers node $k+1$, and we set $d=k+1$.
Using $s_{k+1}^u$ we check if $s_{k+2}^u < 1$ for coverage of node $k+2$, and so on.
In essence, the benefit of sampling node $k$ is propagated \textit{downstream} through a chain of connected nodes $j$ along SPG $\cG$ via scalars $s_j^u > 1$.
We summarize our upstream/downstream procedure in Algorithm~\ref{alg:grd}.

\begin{algorithm}[t]
	\caption{Greedy Elaborate Path Graph Sampling ($\beta=2$)}\label{alg:grd}
\begin{algorithmic}
		\STATE Initialize \: $\Omega \leftarrow \cV$, $\cS \leftarrow \emptyset$, $q\leftarrow 1$
\STATE Initialize \: $\W^q\leftarrow \W$, \textbf{Flag} $\leftarrow$ False
		\WHILE{$\Omega \neq \emptyset$ and $q \leq K$}

		\STATE $\cS_k \leftarrow$ Take the procedure's steps 1-4 using \eqref{eq:sl}, \eqref{eq:su} to select sample node $k$.
		\STATE $\Omega_k \leftarrow$ Given $\cS_k$, compute $d$ via procedure in \ref{subsubsec:computing_scalars}.
		\STATE $\cS \leftarrow \cS \cup \cS_k$,  $\Omega \leftarrow \Omega \setminus \Omega_k$
		\STATE Update $\cG$ by remaining nodes in $\Omega$.
		\STATE $q \leftarrow q+1$
		\ENDWHILE
		\IF{$\Omega = \emptyset$}
		\STATE \textbf{Flag} $\leftarrow$ True
		\ENDIF
	\end{algorithmic}
\end{algorithm}
 
The next lemma states that the scalars computed during the upstream/downstream procedure can indeed compose a diagonal matrix $\S^d$ so that the Gershgorin disc left-end condition $\lambda_{\min}^-(\S^d(\diag{\e_k} + \mu \cL^d)(\S^d)^{-1}) \geq T$ in problem \eqref{eq:sample_obj3} is indeed satisfied.
\begin{lemma}\label{lemma:gctsl}
Using $s_{i+1}^l$ computed using \eqref{eq:sl} for  nodes $i+1 \in \{1, \ldots, k-1\}$,  $s_{i+1}^u$ computed using \eqref{eq:su} for sample node $k$, and $s^u_{i+1}$ computed using \eqref{eq:sd} for nodes $i+1 \in \{k+1, \ldots, d\}$ to compose $\S^d$, $\lambda^-_{\min}(\S^d (\diag{\e_k} + \mu \cL^d) (\S^d)^{-1}) \geq T$.
\end{lemma}

The proof of Lemma~\ref{lemma:gctsl} is provided in Appendix~\ref{appendix:lemmaproof}.

\begin{figure}[t]
	\centering
	\includegraphics[width=0.99\linewidth]{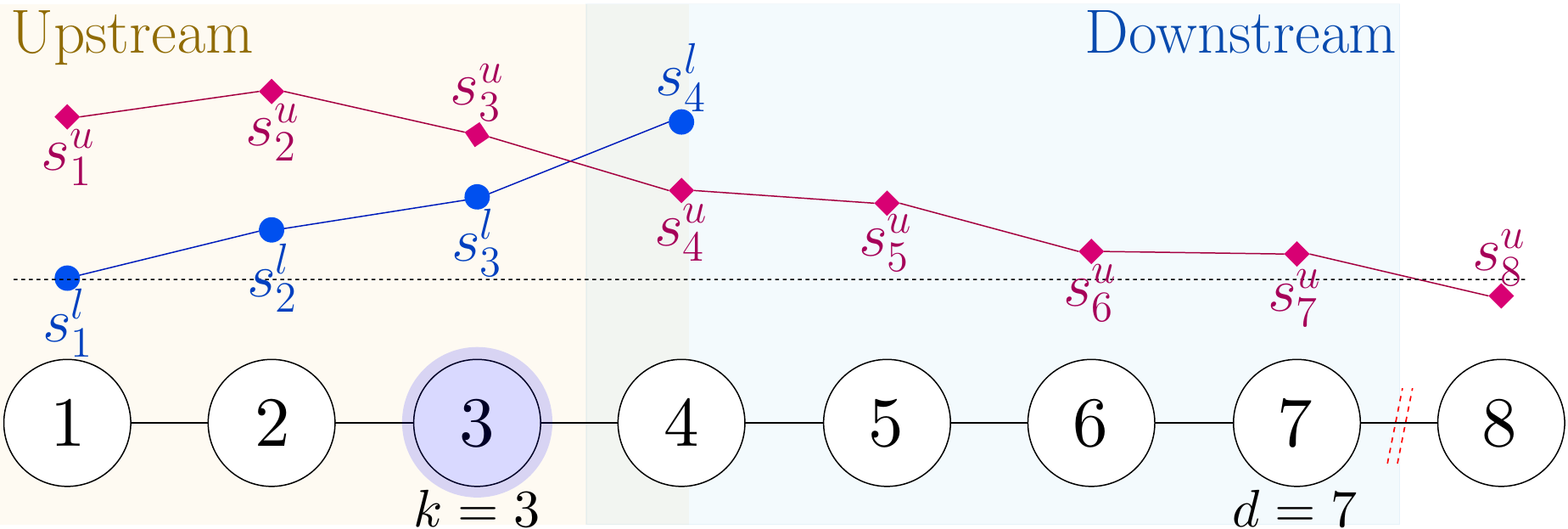}
	\vspace{-0.1in}
	\caption{Illustration of \textit{Upstream} and \textit{Downstream} procedures in an SPG graph.
		The algorithm iteratively computes the scalars $s^u_i$ and $s^l_i$ until $s^u_i < s^l_i$, identifying the sampled node ($k=3$).
		Subsequently, the downstream procedure computes the coverage of sampled node $k$ onwards.
		Covered nodes are then removed from the graph, allowing the sampling algorithm to proceed with the remaining SPG.
	}
	\label{fig:updownstream}
\end{figure}

\subsection{Sampling for Simple Path Graph with Self-loops}
\label{subsec:sampling_selfLoops}

In the case of $\beta=0$ in the graph unfolding procedure (see Fig.\;\ref{fig:vsg2}c), edge sparsification introduces self-loops in the resulting SPG $\cG$.
Self-loops in a graph complicate our developed upstream sampling algorithm by varying the initial Gershgorin disc left-ends of Laplacian matrix $\cL$.

Thanks to a recent theorem called \textit{Gershgorin Disc Perfect Alignment}  (GDPA)  \cite{yang2021signed}, we can resolve this problem by first aligning Gershgorin disc left-ends at minimum eigenvalue $\lambda_{\min}(\cL)$ before executing our upstream sampling algorithm.
Specifically, GDPA states that a similarity-transformed matrix $\cL'=\P\cL\P^{-1}$, where $\P = \diag{1/v_{1,1}, \ldots, 1/v_{1,N}}$ and $\v_1$ is $\cL$'s first strictly positive eigenvector, has Gershgorin disc left-ends perfectly aligned at $\lambda_{\min}(\cL)$.
Computing the first eigenvector $\v_1$ for our sparse and symmetric Laplacian matrix $\cL$ can be done in linear time using \textit{Locally Optimal Block Preconditioned Conjugate Gradient} (LOBPCG)~\cite{knyazev2007block}.

Although the first eigenvector $\v_1$ of a generalized Laplacian $\cL$ for a positive graph is provably strictly positive  \cite{yang2021signed}, in practice, the computation of $\v_1$ can be unstable when $v_{1,i} \approx 0, \exists i$.
This makes the perfect alignment of disc left-ends at $\lambda_{\min}(\cL)$ difficult.
We circumvent the need to compute $\v_1$ directly as follows.
We know that disc left-ends of similarity-transformed matrix $\P\cL\P^{-1}$ are aligned at $\lambda_{\mathrm{min}}(\cL)$, \ie, 
\begin{align}\label{eq:PLP1}
\left(\P \cL \P^{-1}\right)\boldsymbol{1} = \lambda_\mathrm{min}\boldsymbol{1}.
\end{align}

Given that $\cL$ is tri-diagonal for an SPG $\cG$, we can write for rows $1$ to $N-1$ of $\P \cL \P^{-1}$: 

\begin{small}
\begin{align}
\lambda_{\min} &= \cL_{1,1} - \frac{p_1}{p_2} |\cL_{1,2}| 
\label{eq:ratios} \\
\lambda_{\min} &= \cL_{i,i} - \frac{p_i}{p_{i-1}} | \cL_{i,i-1}| - \frac{p_i}{p_{i+1}} |\cL_{i,i+1}|, ~~i \in \{2, \ldots, N-1\}
\nonumber 
\end{align}
\end{small}
If we now define $\alpha_i \triangleq \frac{p_i}{p_{i+1}} = \frac{v_{1,i+1}}{v_{1,i}}$, then \eqref{eq:ratios} simplifies to
\begin{align}
\lambda_{\min} &= \cL_{1,1} - \alpha_1 |\cL_{1,2}| 
\label{eq:ratios2} \\
\lambda_{\min} &= \cL_{i,i} - \alpha_{i-1}^{-1} | \cL_{i,i-1}| - \alpha_i |\cL_{i,i+1}|, ~~i \in \{2, \ldots, N-1\}
\nonumber 
\end{align}
Using \eqref{eq:ratios2}, $\alpha_i$ can be computed iteratively from $i=1$ onwards. 
After disc left-ends of $\P \cL \P^{-1}$ are aligned at $\lambda_{\min}(\cL)$, our developed upstream sampling algorithm can be deployed.

\subsection{Computational Complexity Analysis}
\label{subsec:GS-Complexity}

To select one sample $k$, graph sampling involves one path graph  partition.
In the upstream procedure, scalars $s^l_i$ and $s^u_i$ are computed once for each node using \eqref{eq:sl} and \eqref{eq:su}, from $i=1$ to $i=k$.
\blue{This incurs a cost of $\cO(k)$.
The downstream procedure similarly incurs a cost of $\cO(d-k)$.
Combining both procedures, the total cost for each path graph partition is $\cO(d)$.
Given $N$ nodes in an SPG, after each partition there are $N-d$ nodes remaining.
Thus, the overall complexity is $\cO(N)$.}

Binary search with precision $\epsilon$ (e.g., $\epsilon=10^{-6}$) is used to determine the appropriate $T$ in each iteration, resulting in a total algorithm cost of $\cO(N \log \frac{1}{\epsilon})$.
Since $\epsilon$ is fixed and not a function of $N$, the overall computational cost is $\cO(N)$.

In the case of SPG with self-loops, the computation complexity of $\lambda_{\min}(\cL)$ using LOBPCG algorithm for sparse and symmetric graph Laplacian matrix $\cL$ is linear.
The iterative alignment of left-end of discs also takes $\cO(N)$, and hence the overall complexity of the algorithm is $\cO(N\log\frac{1}{\epsilon})$, same as the case for SPG without self-loops.

\section{Experiments}\label{sec:results}

We first compare the signal reconstruction quality of our graph unfolding procedure and sampling algorithm with several existing graph sampling methods. 
To assess the performance of video summarization, we employ several datasets, including VSUMM/OVPs and YouTube~\cite{deavila2011vsumm}.
Notably, given our focus on keyframe selection for transitory videos, we created a new dataset (see Section~\ref{mhsd}) to address the current problem of lacking keyframe-based video datasets on a larger scale and with flexible licensing. 
Finally, we present an ablation study on algorithm components.

\subsection{Graph Sampling Performance}\label{subsec:GSPerf}

We show the efficacy of our graph unfolding procedure and sampling method specialized for $M$-EPG graphs.
We conducted a signal-reconstruction experiment similar to \cite{bai2020fast}, using graphs generated from videos in the VSUMM~\cite{deavila2011vsumm} dataset.
\blue{Our sampler is agnostic to the choice of graph signal, presuming only smoothness \wrt{} $\cL$; we therefore assess reconstruction on scalar signals, following the benchmark of \gdas{}~\cite{bai2020fast}.
}
The graph construction process, detailed in Section~\ref{sec:graph}, generated 25 $2$-EPG graphs at $2$ frames per second (fps), resulting in graph sizes between $N=116$ and $N=230$.
For edge weight computation (see \eqref{eq:edgeWeight}), we used $\sigma=6$.

To simulate graph signals for evaluation, we considered two classes of random signals, as done in \cite{bai2020fast}:
\begin{itemize}
\item \textbf{BL Signal:} The ground-truth was a bandlimited (BL) signal defined by the first $\lfloor N/20 \rfloor$ eigenvectors of Laplacian $\L$, resulting in a strictly low-pass signal.
The GFT coefficients are drawn from a $\mathcal{N}(1,0.5^2)$ distribution~\cite{anis2016efficient}.
\item \textbf{Gaussian Markov Random Field (GMRF):} We employed a shifted version of Laplacian $\L$ as the inverse covariance (precision) matrix, \ie, a multivariate Gaussian distribution $\mathcal{N}\left(\mathbf{0}, (\L + \delta \I)^{-1}\right)$, where $\delta=10^{-4}$, same as \cite{bai2020fast}. 
To ensure consistency of graph signal power, we normalized the signal power using $(\x - \bar{\x})/\text{std}(\x)$, where $\bar{\x}$ and $\text{std}(\x)$ denote the mean and standard deviation of the graph signal, respectively.
\end{itemize}

To compute the reconstructed signal, we solved linear system \eqref{eq:reconObj} to reconstruct the signal from (noisy) observations and measured the \textit{mean squared error} (MSE) for both $\lfloor N/20 \rfloor$-BL and GMRF graph signals (see Fig.~\ref{fig:mse_vs_k}).

\begin{figure*}
    \centering
    \begin{minipage}{0.49\linewidth}
				\subfloat[\label{fig:mse-a}]{\includegraphics[width=0.80\linewidth]{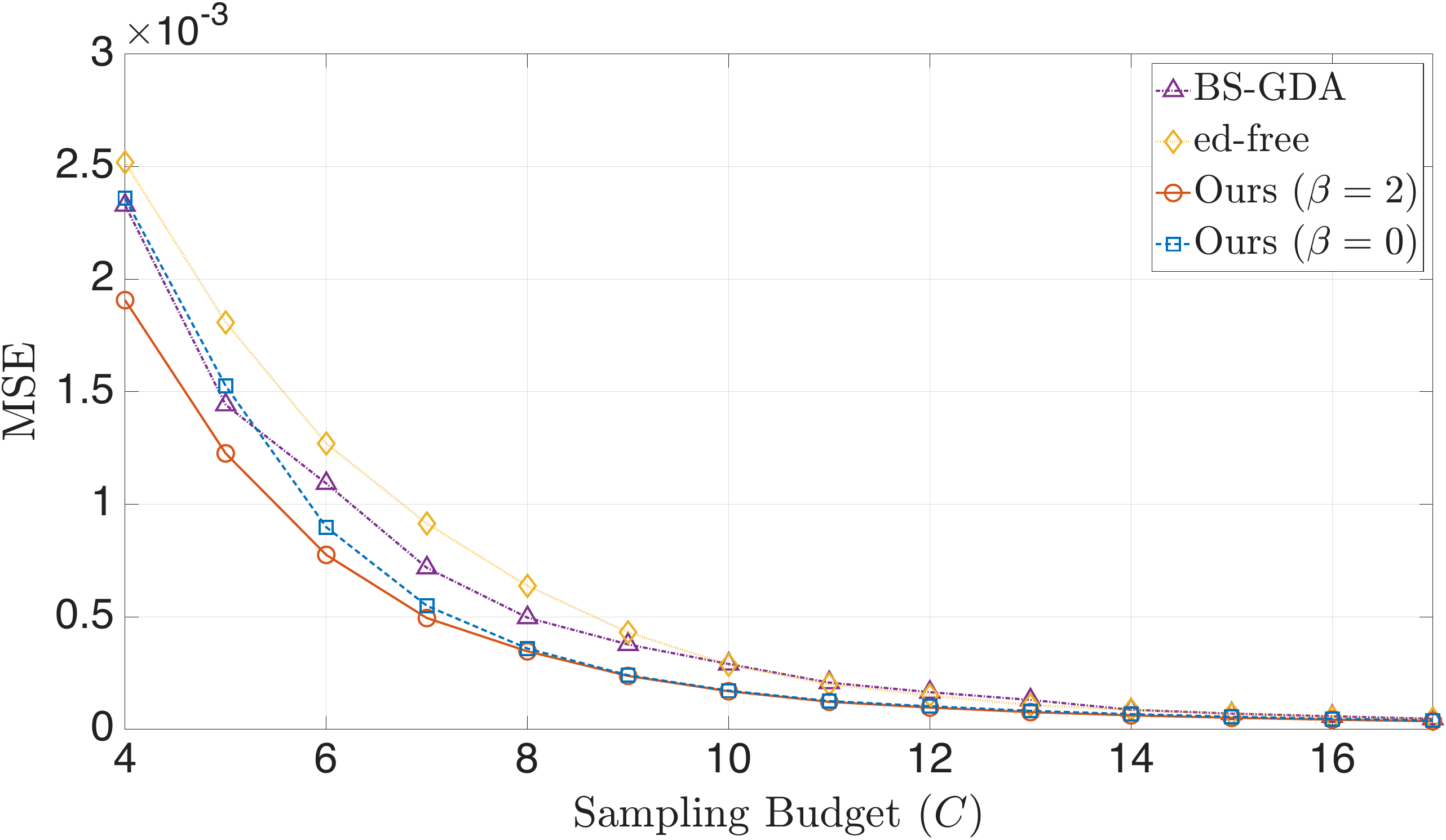}
        }\hfill
    \end{minipage}\begin{minipage}{0.49\linewidth}
				\subfloat[\label{fig:mse-b}]{\includegraphics[width=0.80\linewidth]{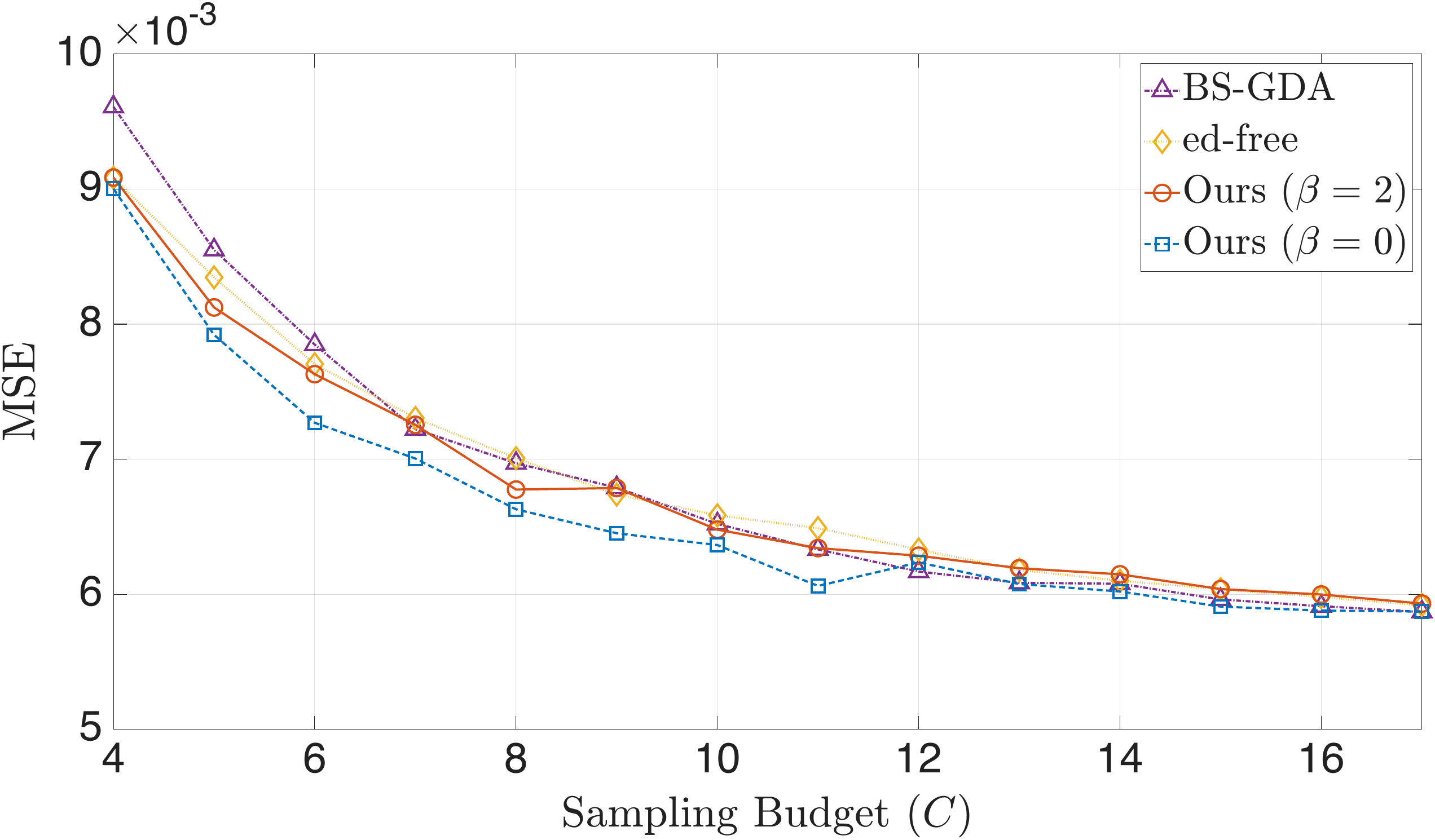}
        }\hfill
    \end{minipage}\\
    \begin{minipage}{0.49\linewidth}
\subfloat[\label{fig:mse-c}]{\includegraphics[width=0.80\linewidth]{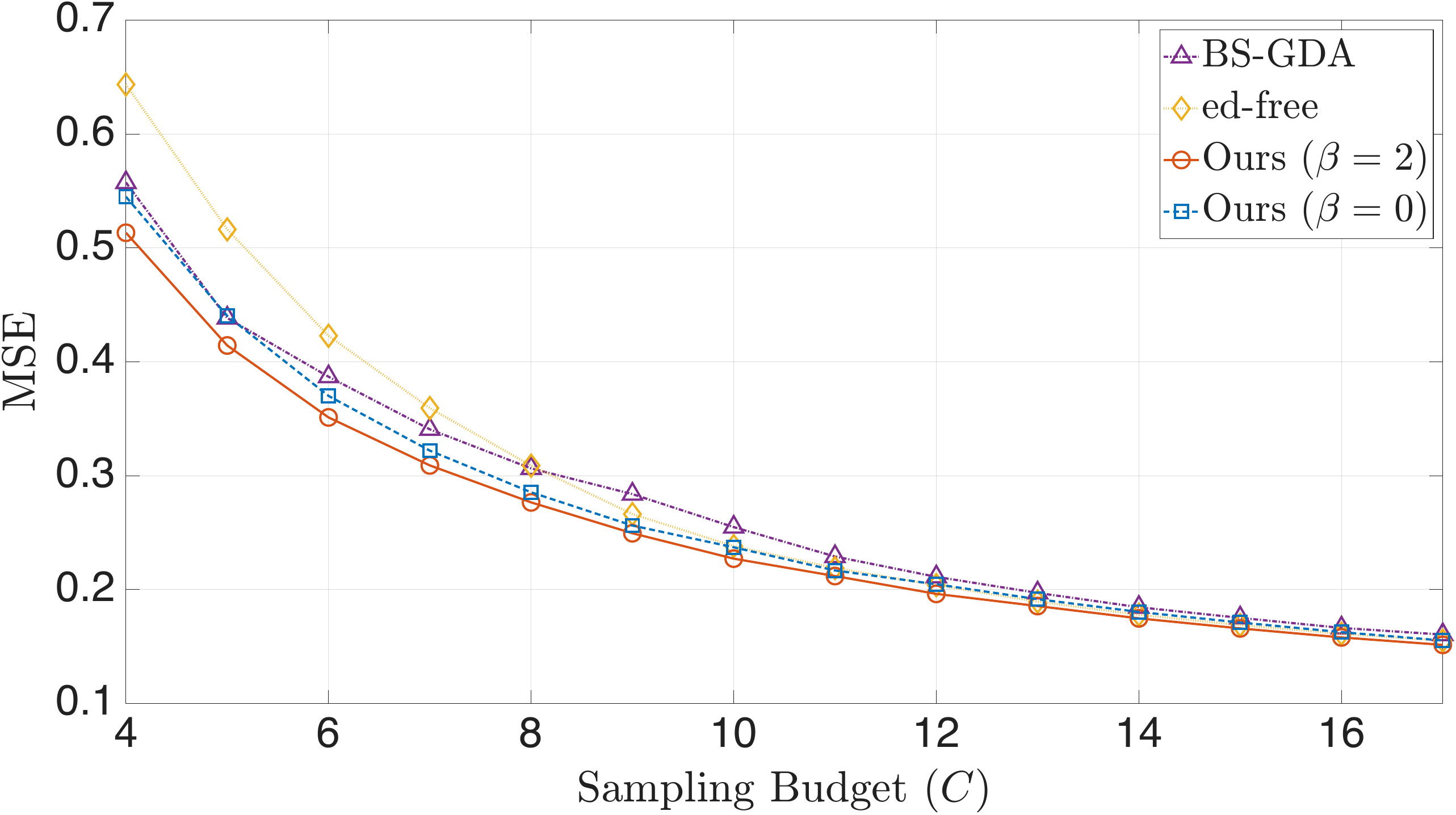}
        }\hfill
    \end{minipage}\begin{minipage}{0.49\linewidth}
\subfloat[\label{fig:mse-d}]{\includegraphics[width=0.80\linewidth]{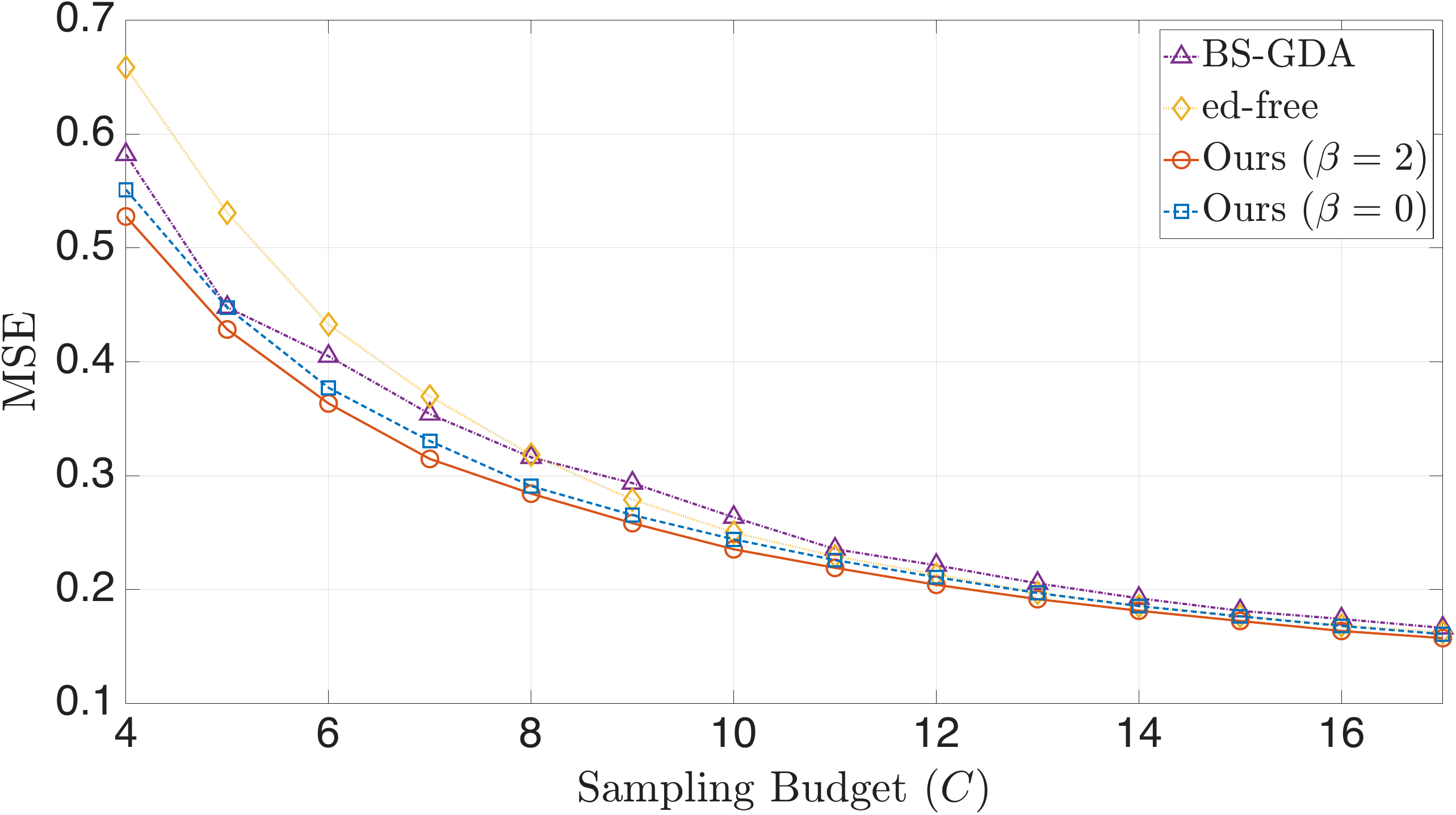}
        }\hfill
    \end{minipage}\\
    \vspace{-0.05in}
    \caption{MSE reconstruction performance of our graph sampling algorithm for different graph signals defined on $2$-EPG graphs using GLR-based signal reconstruction in comparison to \gdas{}~\cite{bai2020fast} and Ed-free~\cite{sakiyama2019edfree}.
        (\ref{fig:mse-a})-(\ref{fig:mse-b}) show results for $\lfloor \frac{N}{20} \rfloor$-BL graph signals with two different noise levels: noise-free and $20$dB, respectively.
        (\ref{fig:mse-c})-(\ref{fig:mse-d}) show results for GMRF graph signals with two different noise levels: noise-free, and $20$dB, respectively.
    }\label{fig:mse_vs_k}
\end{figure*}

We conducted a Monte Carlo experiment on the 25 video graphs, repeating each simulation 100 times.
We tested two noise settings: noise-free signals (Fig.~\ref{fig:mse-a} and \ref{fig:mse-c}), and $20$dB noisy signals (Fig.\;\ref{fig:mse-b} and \ref{fig:mse-d}).

As competing schemes, we used \cite{bai2020fast} and \cite{sakiyama2019edfree} on $M$-EPG graphs.
For our scheme, we first unfolded each EPG graph using one of two unfolding procedures ($\beta=2$ and $\beta=0$).
We then applied the two variants of our algorithm corresponding to unfolded graphs with and without self-loops.
In this experiment, we set the numerical precision of binary search and the regularization parameter in \gdas-based variants to $\epsilon=10^{-9}$ and $\mu=0.05$, respectively.
For Ed-free~\cite{sakiyama2019edfree}, we used the recommended parameters in the experiments in \cite{bai2020fast}, including the Chebyshev polynomial approximation parameter of 12 and the filter width parameter of $\nu=12$.

For BL signals, when the noise level was low, the $\beta=2$ variant, which unfolded the graph to an SPG without self-loops, performed best.
At higher noise levels, the $\beta=0$ variant, which unfolded the graph to an SPG with self-loops, performed best.
In the case of GMRF signals, the $\beta=2$ variant maintained the best performance.

Our two unfolding variants provide different bounds on the GLR term in \eqref{eq:GLR_bound}.
In general, the $\beta=2$ variant offers a tighter upper bound for smooth signals, leading to better performance. Notably, for perfectly smooth signals ($x_i=x_j$ for all $(i,j)\in \cE$), $\x^\top \cL\x = 0$ for $\beta=2$ (thus the tightest upper bound possible), but remains nonzero for $\beta=0$ due to self-loops.
Thus, the $\beta=2$ variant exhibits a slight reconstruction advantage, although both variants outperform competitors.

\subsection{Keyframe Selection Performance}
\label{subsec:VSPerf}

\subsubsection{VSUMM/OVP dataset}\label{vsummovp}

VSUMM \cite{deavila2011vsumm} is the most popular dataset for keyframe-based video summarization in the literature. 
Collected from \textit{Open Video Project}\footnote{\url{https://open-video.org/}} (OVP)~\cite{mundur2006keyframebased}, it consists of 50 videos of duration $1$-$4$ minutes and resolution of $352 \times 240$ in the MPEG-1 format (mainly in 30fps).
The videos are in different genres: documentary, educational, lecture, historical, etc.

Fifty human subjects participated in constructing user summaries, each annotating five videos. 
Thus, the dataset has five sets of keyframes as ground-truth per video, and each set may contain a different number of keyframes. 
Here, we refer to computer-generated keyframes 
as \textit{automatic summary} ($\cA$), and the human-annotated keyframes as \textit{user summary} $\cU_u$, for each user $u \in \{1,\ldots, 5\}$.

For each video, the agreements between $\cA$ and each of the user summaries, $\mathcal{U}_u$, were evaluated using \textit{Precision} ($P_u$), \textit{Recall} ($R_u$), and F$_{1,u}$.
We define the precision, recall, and F$_1$ against each human subject as follows \cite{ma2020similarity,mei2021patcha},
\begin{align*}
	P_{u} = \frac{|\mathcal{A} \cap \mathcal{U}_u|}{|\mathcal{A}|},\quad 
	R_{u} = \frac{|\mathcal{A} \cap \mathcal{U}_u|}{|\mathcal{U}_u|}, \quad
	\mathrm{F}_{1,u} = \frac{2P_u R_u}{P_u+R_u}
\end{align*}
where $|\cA \cap \cU_u|$ is the number of keyframes in $\cA$ matching  $\cU_u$.
$|\cA|$ and $|\cU_u|$ denote the number of selected keyframes in the automatic and user summaries, respectively.
For each video, we measured the mean precision, recall, and F$_1$ across all users. 
The values in Table\;\ref{tab:resvsumm} are the averages of these measures across all videos in the dataset.
Our evaluation follows the \blue{one-to-one keyframe-matching} protocol of VSUMM~\cite{deavila2011vsumm, mei2021patcha}: \blue{each selected keyframe in the automatic summary $\cA$ is matched to at most one ground-truth keyframe in a user summary $\cU_u$ (a bipartite one-to-one matching),} where two frames are deemed to match if they are similar in appearance~\cite{deavila2011vsumm} and within a temporal distance of $2.5$ seconds~\cite{ma2020similarity,mei2021patcha}.
Unlike \cite{ma2022graph,ma2020similarity}, we did not conduct exhaustive searches for optimal keyframe numbers on a per-video basis.
Instead, we employed a fixed, video-invariant sampling budget, $C$, and tuned its value to maximize the F$_1$ score.

For comparison, we followed \cite{deavila2011vsumm,mei2021patcha} and compared our method with DT~\cite{mundur2006keyframebased}, VSUMM~\cite{deavila2011vsumm}, STIMO~\cite{furini2009stimo}, MSR \cite{mei2015video}, AGDS \cite{cong2017adaptive} and SBOMP \cite{mei2021patcha}.
The results for DT, STIMO and VSUMM methods are available from the official VSUMM  website\footnote{\url{http://www.sites.google.com/site/vsummsite/}}.
Since the other algorithm implementations or their automatic summaries are not publicly available, we relied on the best results reported in~\cite{mei2021patcha} for comparison.

\begin{table*}\centering
\caption{Results on VSUMM. Human performance is computed by evaluating each annotator against the remaining four.}\label{tab:resvsumm}
\begin{threeparttable}
\begin{tabular}{llllll}
		\toprule
		Algorithm & $P$ (\%) & $R$ (\%) & F$_1$ (\%) \\
		\midrule
        Human  & $57.15 \pm 14.41$ &  $57.22 \pm 17.00$ &  $55.68 \pm 14.17$ \\
		\midrule
		DT \cite{mundur2006keyframebased} & 35.51 & 26.71 & 29.43 \\
		STIMO \cite{furini2009stimo} & 34.73 & 40.03 & 35.75 \\
		VSUMM$^*$ \cite{deavila2011vsumm} & \textbf{\underline{47.26}} & 42.34 & 43.52 \\ MSR \cite{mei2015video} & 36.94 & 57.61 & 43.39 \\
		AGDS \cite{cong2017adaptive} & 37.57 & 64.60 & 45.52 \\
		SBOMP \cite{mei2021patcha}  & 39.28 & 62.28 & 46.68 \\
		SBOMPn \cite{mei2021patcha} & \textbf{41.23} & 68.47 & \underline{49.70} \\
GDAVS \cite{sahami2022fast}	& 39.67		& 72.48	& 48.92 \\
		\midrule
		\blue{SUM-GAN$^\dagger$~\cite{mahasseni2017unsupervised}}				& \blue{$17.53\pm 8.33$}		& \blue{$31.79\pm 12.59$}	& \blue{$21.68\pm 8.63$} \\
\blue{AC-SUM-GAN$^\dagger$~\cite{apostolidis2021acsumgan}}				& \blue{$18.72\pm 7.70$}		& \blue{$35.04\pm 14.09$}	& \blue{$23.44\pm 8.53$} \\ DPP-LSTM\blue{$^\dagger$}~\cite{zhang2016videob}				& $19.38\pm 10.40$		& $34.15\pm 13.41$	& $23.76\pm 10.66$ \\
		VASNet\blue{$^\dagger$}~\cite{fajtl2019summarizing}				& $17.13\pm 7.30$		& $32.66\pm 14.40$	& $21.53\pm 8.18$ \\
		MSVA\blue{$^\dagger$}~\cite{ghauri2021supervisedb}				& $18.49\pm 8.75$		& $33.91\pm 15.61$	& $23.01\pm 9.97$ \\
		SSPVS\blue{$^\dagger$}~\cite{li2023progressivea}				& $16.93\pm 9.93$		& $30.18\pm 14.22$	& $20.79\pm 10.36$ \\
		CSTA\blue{$^\dagger$}~\cite{son2024cstaa}				& $19.62\pm 7.76$		& $36.80\pm 14.96$	& $24.55\pm 8.51$ \\
		\midrule Ours ({\tiny $M=2$, $\beta=0$, $\mu=0.04$, $\sigma=7.95$}) 				& $36.63\pm 12.37$		& \underline{$\mathbf{83.88\pm 11.61}$}	& $49.06\pm 10.97$ \\ Ours ({\tiny $M=3$, $\beta=0$, $\mu=0.04$, $\sigma=7.95$}) 				& $35.57\pm 12.70$		& $\mathbf{81.25\pm 12.41}$	& $47.56\pm 11.33$ \\ \midrule
		Ours ({\tiny $M=1$, $\mu=0.03$, $\sigma=1.41$}) 				& $39.08 \pm 12.49$		& $71.69 \pm 10.87$	& $48.54\pm 10.06$ \\ Ours ({\tiny $M=2$, $\beta=2$, $\mu=0.03$, $\sigma=1.41$}) 				& \underline{$40.77 \pm 13.24$}		& \underline{$74.16 \pm 8.95$}	& \underline{$\mathbf{50.54 \pm 10.05}$} \\ Ours ({\tiny $M=3$, $\beta=2$, $\mu=0.03$, $\sigma=1.41$}) 				& $40.21\pm 13.67$		& $73.13\pm 9.49$	& $\mathbf{49.81\pm 10.61}$\\ \bottomrule
\end{tabular}
\begin{tablenotes}[flushleft]
\footnotesize
\item[$\dagger$] \blue{Cross-paradigm model designed for keyshot frame-overlap.}
\end{tablenotes}
\end{threeparttable}
\vspace{-0.1in}
\end{table*}

\textit{Adaptation of importance score based models.}
The supervised importance-based models DPP-LSTM~\cite{zhang2016videob}, MSVA~\cite{ghauri2021supervisedb}, CSTA~\cite{son2024cstaa} and SSPVS~\cite{li2023progressivea} generate continuous frame-level importance scores suitable for shot-based evaluations like on TVSum~\cite{yalesong2015tvsuma} and SumMe~\cite{gygli2014creating}.
To incorporate them into our unsupervised keyframe-based protocol (and later YouTube and \ourdataset{}), we use their publicly released pre-trained models, selecting the best-performing checkpoint. We compute per-frame importance scores for each test video.
We then derive a keyframe set by ranking frames and selecting the highest-scoring ones while enforcing a temporal exclusion interval to suppress redundant adjacent selections.
For SSPVS, we note that its released pre-trained model includes OVP and YouTube data as part of its augmented training, and therefore it is not strictly unsupervised in contrast to the other methods.

For VASNet, the pre-trained model referenced in the original repository is no longer accessible.
Consequently, we train VASNet from scratch following the standard hyper-parameter configuration provided in the official implementation.
We train the network using the Adam optimizer with a learning rate of $5 \times 10^{-5}$ and $L_{2}$ regularization (weight decay) of $1 \times 10^{-5}$.
Training is performed for 300 epochs with a batch size of~1 while processing full video sequences.

\blue{We further compare against the unsupervised adversarial summarizers SUM-GAN~\cite{mahasseni2017unsupervised} and the more recent AC-SUM-GAN~\cite{apostolidis2021acsumgan}.
As SUM-GAN has no official implementation, we adopt a widely used public re-implementation\footnote{\url{https://github.com/j-min/Adversarial_Video_Summary}}; its selector network outputs frame-level importance scores, which we convert to keyframes as above.
Both are trained for 100 epochs, except SUM-GAN on \ourdataset{}, which requires 500 epochs to converge.
For AC-SUM-GAN, we report results from the epoch identified by its label-free criterion under the best-performing length regularization parameter.
}

Table\;\ref{tab:resvsumm} shows the experimental results, where the numbers in \underline{\textbf{underlined bold}}, \textbf{bold}, and \underline{underlined} indicate the top three performances.
For our method, we report the standard deviation across runs on each dataset. 
For competing methods, such statistics are not provided in the literature; hence, we report their published mean values only.
The significantly higher precision of the VSUMM method~\cite{deavila2011vsumm} can be attributed to its preprocessing steps, which eliminate low-quality frames, and its post-processing pruning, which removes very similar selected frames--techniques not employed by other methods.
SBOMP~\cite{mei2021patcha} is a recently proposed video summarization method with SOTA performance and relies on sub-frame representation by extracting features for each image patch collected in a \textit{matrix block}.
This results in a complex representation for video frames in their dictionary learning-based solution, increasing the algorithm's complexity---$\cO(Cp^2\ell N^2 + p^3 C^3 N)$, where $\ell$ is the dimension of feature vector, $p$ is the number of patches per frame, and $N,C$ are the number of frames and keyframes, respectively\footnote{We assume convergence within constant $t\ll N$ (independent of $N,C$).}.
While $\ell$ is fixed, the number of keyframes is typically proportional to the number of frames, i.e., $C \propto N$.
This effectively increases the complexity to $\cO(N^4)$.
The SBOMPn version extends the representation to the temporal dimension by considering \textit{super-blocks} of patches from adjacent frames, thereby increasing the complexity.

In contrast, our graph sampling complexity is $\cO(N \log_2 \frac{1}{\epsilon})$, where $\epsilon$ is the binary search precision.
Our method surpassed SBOMP and its more complex variant SBOMPn, achieving comparable performance at significantly reduced complexity.

\subsubsection{YouTube Dataset}\label{youtube}

YouTube~\cite{deavila2011vsumm} serves as a dataset comparable to VSUMM/OVP, specifically designed for keyframe-based video summarization. 
The 50 videos, all sourced from YouTube, span genres like cartoons, news, commercials, and TV shows. 
Their durations range from 1 to 10 minutes.

\begin{table*}\centering
\caption{Results on YouTube. Human performance is computed by evaluating each annotator against the remaining four.}\label{tab:youtube}
\begin{threeparttable}
\begin{tabular}{lllll}
			\toprule
			Algorithm & $P$ (\%) & $R$ (\%) & F$_1$ (\%) & $\overline{C}$ \\
			\midrule
			Human &  $49.01\pm 19.31$ & $49.24\pm 19.39$ & $47.52\pm 19.35$ & 10.7 \\
			\midrule
			VSUMM$^*$ \cite{deavila2011vsumm} &  \underline{\textbf{43.11}} & 43.56 & \textbf{42.38} & 10.3 \\
			\blue{MSR~\cite{mei2015video}} & \blue{39.27} & \blue{49.09} & \blue{39.23} & \blue{14.3} \\
			SMRS \cite{elhamifar2012see} & 33.84 & 54.92 & 39.62 & 17.3 \\
			SSDS \cite{wang2017representative} & 33.45 & 47.10 & 36.98 & 13.0 \\
			AGDS \cite{cong2017adaptive} & 34.00 & \textbf{58.08} & 40.65 & 16.0 \\
			NSMIS~\cite{dornaika2018instance} & 35.61 & \underline{\textbf{59.95}} & 41.27 & 18.7 \\
			GCSD \cite{ma2022graph}  &  \underline{37.28} & 54.33 & \underline{42.20} & 13.7 \\
			\midrule
            \blue{SUM-GAN$^\dagger$~\cite{mahasseni2017unsupervised}} & \blue{$19.24\pm 10.58$}		& \blue{$29.59\pm 18.70$}	& \blue{$22.07\pm 12.81$} & \blue{13.7} \\
\blue{AC-SUM-GAN$^\dagger$~\cite{apostolidis2021acsumgan}} & \blue{$21.88\pm 14.28$}		& \blue{$31.21\pm 23.58$}	& \blue{$23.84\pm 16.31$} & \blue{12.5} \\ DPP-LSTM\blue{$^\dagger$}~\cite{zhang2016videob} & $20.97\pm 12.18$		& $29.59\pm 20.26$	& $23.00\pm 13.54$ & 12.5 \\
            VASNet\blue{$^\dagger$}~\cite{fajtl2019summarizing} & $22.05\pm 12.38$		& $28.65\pm 17.49$	& $23.29\pm 12.81$ & 11.7 \\
            MSVA\blue{$^\dagger$}~\cite{ghauri2021supervisedb} & $21.80\pm 13.63$		& $29.46\pm 19.64$	& $23.89\pm 15.29$ & 12.0 \\
            SSPVS\blue{$^\dagger$}~\cite{li2023progressivea} & $20.98\pm 10.95$		& $28.50\pm 19.96$	& $22.57\pm 13.09$ & 11.7 \\
            CSTA\blue{$^\dagger$}~\cite{son2024cstaa} & $27.41\pm 12.76$		& $33.72\pm 19.70$	& $28.04\pm 13.95$ & 11.1 \\ 
            \midrule
			SBOMP~\cite{mei2021patcha} ({\tiny $P=1$, $E=0.05$})			& $32.09 \pm 16.03$		& $44.06 \pm 24.62$	& $35.45 \pm 17.97$ & 13.4 \\
SBOMP~\cite{mei2021patcha} ({\tiny $P=5$, $E=0.05$})			& $32.81 \pm 19.62$		& $45.99 \pm 24.98$	& $36.52\pm 20.62$ & 14.3 \\
			\midrule
Ours ({\tiny $M=2$, $\beta=0$, $\mu=0.25$, $\sigma=6$}) 				& $36.87\pm 19.50$		& $51.11\pm 24.95$	& $40.89\pm 20.43$ & 13.9 \\ Ours ({\tiny $M=2$, $\beta=2$, $\mu=0.006$, $\sigma=6$}) 				& $\mathbf{37.62 \pm 19.36}$		& \underline{$56.97 \pm 26.09$}	& \underline{$\mathbf{42.99 \pm 20.60}$} & 15.0 \\ \bottomrule
\end{tabular}
\begin{tablenotes}[flushleft]
\footnotesize
\item[$\dagger$] \blue{Cross-paradigm model designed for keyshot frame-overlap.}
\end{tablenotes}
\end{threeparttable}
\vspace{-0.1in}
\end{table*}

Unlike VSUMM~\cite{deavila2011vsumm}, the YouTube dataset is less commonly used in the literature.
In this experiment, we utilized the SOTA approach by \cite{ma2022graph}, which has recently demonstrated promising results on the dataset.
In contrast to our previous experiment, where we employed a fixed budget ratio for all videos, according to the protocol in \cite{ma2022graph}, the optimal keyframe number for each video is exhaustively searched; specifically, we searched the optimal keyframe number for each video to maximize F$_1$, varying from 6 to 20 with a step size of 2, as done in \cite{ma2022graph}.
The evaluation process mirrors the \blue{one-to-one keyframe-matching} protocol of Section\;\ref{vsummovp}, encompassing precision, recall, and F$_1$ measure.
Further, the benchmark yields the average budget (number of keyframes) across all videos, denoted as $\overline{C}$.
Consistent with \cite{ma2022graph}, we set the temporal matching criterion to 2 seconds.

Table~\ref{tab:youtube} summarizes the evaluation results, including parameter configurations.
For SBOMP~\cite{mei2021patcha}, we used the official implementation with two patching strategies: a single patch for the entire image and the $4+1$ patching strategy from the paper.
Following the recommended strategy, we generated sub-frame features for each patch and performed a parameter sweep on error thresholds (E) to report the best results.
Notably, \cite{ma2022graph} builds upon \cite{ma2020similarity} by incorporating an adjacency matrix into dictionary selection, increasing computational complexity, similar to \cite{mei2021patcha}.
For YouTube, we report standard deviations for our method, the SBOMP baseline, and \blue{the deep-learning baselines}, while other results are taken as reported in prior work.
Our method achieved superior performance with lower complexity, even outperforming VSUMM$^*$\cite{deavila2011vsumm}, which relies on preprocessing to filter out low-quality frames and post-processing to prune redundant keyframes.

\subsubsection{\ourdataset{} Dataset}\label{mhsd}

Creating datasets for video analysis is challenging due to its labor-intensive and time-consuming nature~\cite{apostolidis2021video}.
Current datasets were largely focused on longer and dynamic video summarization tasks, often neglecting the distinctive traits of short-form content prevalent on online platforms~\cite{zhang2023measurement}.
To bridge this gap, we developed a new dataset, \ourdataset{}\footnote{\ourdatasetinfull{}}, specifically for transitory video summarization, with a primary emphasis on short-form content.

\vspace{0.05in}
\noindent
\textbf{Dataset Collection}:
For this study, we gathered videos from two major video-sharing platforms, Vimeo and YouTube, using their respective filtering tools that support licensing criteria.
Our selection query criteria were informed by prior research on popular video categories such as \textit{Entertainment}, \textit{DIY}, \textit{Comedy}, etc.~\cite{cheng2013understanding,mcclanahan2017interplay}, aiming to mirror their content distribution.
To facilitate research and ensure ease of use, we focused on videos licensed under \textit{Creative Commons} and its derivatives.
Additionally, we prioritized videos based on popularity, determined by view counts.
Our dataset consists of videos with a duration ranging from a minimum of $45$ seconds to a maximum of $4$ minutes, with an average duration of $110$ seconds.

\vspace{0.05in}
\noindent
\textbf{Annotation}:
We annotate a dataset of $117$ videos, more than doubling the size of both VSUMM and YouTube datasets. 
Annotation was conducted using the Labelbox platform\footnote{https://labelbox.com}, an online tool designed for data annotation.
To maintain consistency in annotations, all three annotators labeled the entire dataset once.
Annotators received initial training through the platform before proceeding with annotation tasks.
Each annotator watched each video in its entirety and selected up to 7 keyframes per video.

\begin{table}\centering
\caption{Results on \ourdataset{}. Human performance is computed by evaluating each annotator against the remaining two.}\label{tab:mhsd}
\begin{threeparttable}
\begin{tabular}{lllll}
			\toprule
			Algorithm & $P$ (\%) & $R$ (\%) & F$_1$ (\%) & $\overline{C}$ \\
			\midrule
			Human Best & 51.62 & 42.95 & 45.44 & 3.9 \\ \midrule
			VSUMM~\cite{deavila2011vsumm} & 31.89 & 50.50 & 37.30 & 7.9 \\ SMRS~\cite{elhamifar2012see} ({\tiny $\alpha=2$, $\tau=10^{-8}$}) & 30.18 & 43.66 & 33.59 & 7.7 \\ SBOMP~\cite{mei2021patcha} ({\tiny $p=1$, $E=0.05$}) & 30.11 & 43.53 & 33.06 & \textbf{7.6} \\ SBOMP~\cite{mei2021patcha} ({\tiny $p=5$, $E=0.05$}) & 30.75 & 43.78 & 33.80 & \underline{7.7} \\ \midrule
            \blue{SUM-GAN$^\dagger$~\cite{mahasseni2017unsupervised}} 		& \blue{$21.44$}		& \blue{$39.14$}	& \blue{$26.69$} & \blue{8.5} \\
\blue{AC-SUM-GAN$^\dagger$~\cite{apostolidis2021acsumgan}} 		& \blue{$35.10$}		& \blue{$45.60$}	& \blue{$36.15$} & \blue{7.1} \\ DPP-LSTM\blue{$^\dagger$}~\cite{zhang2016videob} 		& $26.76$		& $42.43$	& $31.10$ & 8.0 \\
            VASNet\blue{$^\dagger$}~\cite{fajtl2019summarizing} 		& $28.85$		& $42.20$	& $32.45$ & 7.7 \\
            SSPVS\blue{$^\dagger$}~\cite{li2023progressivea} 		& $33.77$		& $45.07$	& $35.68$ & 7.5 \\
            CSTA\blue{$^\dagger$}~\cite{son2024cstaa} & $30.14$		& $47.41$	& $35.59$ & 7.8 \\
\midrule
			Ours ({\tiny $M=1$, $\mu=0.2$, $\sigma=2$}) & \textbf{39.79} & \underline{57.91} & \underline{44.91} & \underline{\textbf{7.5}} \\ Ours ({\tiny $M=2$, $\beta=0$, $\mu=0.2$, $\sigma=4$}) & \underline{38.81} & \textbf{58.74} & \textbf{45.02} & \underline{7.7} \\ Ours ({\tiny $M=2$, $\beta=2$, $\mu=0.2$, $\sigma=4$}) & \underline{\textbf{40.32}} & \underline{\textbf{59.24}} & \underline{\textbf{45.37}} & \underline{\textbf{7.5}} \\ 

			\bottomrule
\end{tabular}
\begin{tablenotes}[flushleft]
\footnotesize
\item[$\dagger$] \blue{Cross-paradigm model designed for keyshot frame-overlap.}
\end{tablenotes}
\end{threeparttable}
\vspace{-0.1in}
\end{table}

\vspace{0.05in}
\noindent
\textbf{Results}:
Table\;\ref{tab:mhsd} presents the performance evaluation of various algorithms on the \ourdataset{} dataset.
Except for \cite{mei2021patcha}, which provides its own implementation, we implemented all other methods for comparison on our new dataset.
All methods use the same CLIP-based feature set $\left\{\f_i\right\}_{i=1}^N$ (see Section~\ref{sec:graph}) except for SBOMP~\cite{mei2021patcha} ($p=5$), where features are generated for each sub-frame patch.
The evaluation protocol followed the guidelines in Section~\ref{youtube}, but we limited the keyframe range to 1 to 11 frames due to the shorter content and fewer annotations.

We report the best performance for each competing scheme after parameter tuning.
For SMRS~\cite{elhamifar2012see}, $\alpha$ and $\tau$ denote the regularization and error tolerance, respectively, with the ADMM optimization iterating 2000 times.
For SBOMP~\cite{mei2021patcha}, $p$ denotes the number of patches based on their recommended strategy.
In this table, VSUMM omits pre- and post-pruning for a fair comparison, unlike VSUMM$^*$ in previous experiments.
Table~\ref{tab:mhsd} summarizes the results and respective parameters, highlighting the top three performances in \underline{\textbf{underlined bold}}, \textbf{bold}, and \underline{underline}.
Our method, with parameters $M=2$ and $\beta=2$, achieved an F$_1$ score of 45.37\%, demonstrating its effectiveness in capturing keyframes in short videos.

\subsection{Ablation Study}\label{subsec:ablation}

The selection of $M$-EPG structure for video representation is grounded in the observation that consecutive video frames exhibit a high degree of similarity, a characteristic central to EPG.
Importantly, this sparse graph structure enables a more efficient sampling algorithm to select keyframes.
Given that it is possible for (short) videos to contain information with long-term dependencies, we performed an ablation study to validate the effectiveness of our chosen combination of EPG construction and the specialized graph sampling algorithm in processing such videos.
The results of different combinations of graph construction (GC) and graph sampling (GS) algorithms are presented in Table\;\ref{tab:ablation1}.

The similarity metric between respective frame feature vectors was computed using our edge computation in \eqref{eq:edgeWeight} with parameter $\sigma=6$.
For both our algorithm and \gdas{}~\cite{bai2020fast}, the precision parameter was set to $\epsilon=10^{-9}$.
This ablation study used the VSUMM dataset and followed the evaluation protocol, including the sampling budget, as outlined in Section\;\ref{vsummovp}.
With $\sigma$, the budget, and $\epsilon$ held fixed across all algorithms, the remaining parameters (shown in Table\;\ref{tab:ablation1}) were tuned by grid search per algorithm, \blue{with each competing method's best F$_1$ reported.
For our method, we instead fix a single $\mu=0.08$ across all $M$, so that the third block isolates the effect of the EPG order $M$;
these numbers are thus not our best, yet our sampler still surpasses the tuned \gdas{}~\cite{bai2020fast} at every $M$.}

\begin{table}\centering \caption{GC/GS Ablation Study ($\sigma=6$, $\epsilon=10^{-9}$)
}\label{tab:ablation1}
\resizebox{1.00\linewidth}{!}{\centering \begin{tabular}{lllll}
			\toprule
			GC& GS& Complexity & F$_1$(\%) \\
			\midrule
			$\tau$-graph ({\tiny$\tau=0.01$, $\mu=0.007$}) & \cite{bai2020fast} & $\mathcal{O}(N^2\ell)$ & $42.40$ \\
			$k$-NN ({\tiny$k=25$, $\mu=0.003$})  & \cite{bai2020fast}  & $\mathcal{O}(N^2k\ell)$ & $42.69$ \\
NNK~\cite{shekkizhar2020graph} ({\tiny$k=50$, $\delta=0.01$, $\mu=0.006$})  & \cite{bai2020fast} & $\mathcal{O}(N^2k\ell+Nk^3)$ & $43.59$ \\
			\midrule
			$1$-EPG ({\tiny$\mu=0.08$}) & \cite{bai2020fast} & $\cO(N\ell + NCP\log_{2}1/\epsilon)$ & $48.55$ \\
			$2$-EPG ({\tiny$\mu=0.08$}) & \cite{bai2020fast} & $\cO(N\ell + NCP\log_{2}1/\epsilon)$ & $49.13$ \\
			$3$-EPG ({\tiny$\mu=0.08$}) & \cite{bai2020fast} & $\cO(N\ell + NCP\log_{2}1/\epsilon)$ & $49.21$ \\
			$4$-EPG ({\tiny$\mu=0.08$}) & \cite{bai2020fast} & $\cO(N\ell + NCP\log_{2}1/\epsilon)$ & $47.92$ \\
\midrule
			$1$-EPG ({\tiny$\mu=\blue{0.08}$})  & Ours & $\mathcal{O}(N\ell + N\log_21/\epsilon)$ & \blue{$50.36$} \\
			$2$-EPG ({\tiny$\beta = 2$, $\mu=\blue{0.08}$})  & Ours & $\mathcal{O}(N\ell + N\log_21/\epsilon)$ & \blue{$50.65$} \\
			$\blue{3}$-EPG ({\tiny$\beta = 2$, $\mu=\blue{0.08}$})  & Ours & $\mathcal{O}(N\ell + N\log_21/\epsilon)$ & \blue{$49.84$} \\
			$\blue{4}$-EPG ({\tiny$\beta = \blue{2}$, $\mu=\blue{0.08}$})  & Ours & $\mathcal{O}(N\ell + N\log_21/\epsilon)$ & \blue{$49.91$} \\
			\bottomrule \end{tabular}\vspace{-0.05in}
}\end{table}

Our sampling algorithm only operates on unfolded graphs; for general graphs, we employed the method described in \cite{bai2020fast}. 
In the first section, we present the results of \gdas{}~\cite{bai2020fast} sampling on general graph construction methods such as $\epsilon$-thresholding, the $k$-NN method, and NNK~\cite{shekkizhar2020graph}. 
These methods typically produce much denser graphs compared to our sparse graph construction, especially with $\tau$-thresholding.
The second section shows the results of \gdas{} applied to our $M$-EPG graph construction scenario. 
Finally, the third section provides the results of our unfolding and sampling applied to the $M$-EPG graph.

The comparison in the first section highlights the effectiveness of our graph construction.
Moreover, the contrast between the second and third sections shows that our new graph sampling algorithm achieved better results for video summarization.

Further, our $M$-EPG graph construction operates with $\cO(MN\ell)$ complexity, and our sampling technique runs in $\cO(N\log_2\frac{1}{\epsilon})$ time, offering improved efficiency in both time (and memory) compared to competitors.
In Table\;\ref{tab:ablation1}, $N$ denotes the number of nodes, and $C$ represents the number of samples.
For \cite{bai2020fast}, $P$ denotes the number of nodes within $p$ hops from a candidate node (note $P=2M$ for \gdas{}).
In the NNK and $k$-NN graph constructions, $k$ signifies the nearest neighbor parameter\footnote{In the literature, an $\cO(N^{1.14})$ approximation of the $k$-NN graph construction exists~\cite{dong2011efficient}, which can benefit both NNK and KNN complexity~\cite{shekkizhar2020graph}.
}.
For the results in the first section of the table, we omit the sampling complexity since the graph construction dominates.

Finally, we conducted an ablation study on several off-the-shelf feature extractors (Table~\ref{tab:ablation_embed}). 
The results show that CLIP (ViT-B/32) yields the best precision, recall, and F$_1$ score. 
Nevertheless, our graph sampling framework is feature-agnostic, as the construction of the similarity graph can be based on any representative semantic descriptor. 
Thus, the method readily adapts to alternative feature choices without modification.

\begin{table}[tbhp]\centering
\caption{Ablation on Choices of Feature Extractor -- VSUMM Dataset. \\ 
 ({\small $\sigma=6/\sqrt{2},\: \text{Budget}=10.4\%, \: \mu=0.03, \: \epsilon=10^{-6}$}) 
}\label{tab:ablation_embed}
\begin{tabular}{lcccc}
    \toprule
    Network & Feature Dim. & $P$ (\%) & $R$ (\%) & F$_1$ (\%) \\
    \midrule
    CLIP (ViT-B/32)~\cite{radford2021learning} & 512 & 37.90 & 86.25 & 50.67 \\
    ResNet-50~\cite{he2016deep} & 2048 & 37.05 & 84.76 & 49.60 \\
    ResNet-18~\cite{he2016deep} & 512 & 36.89 & 83.87 & 49.28 \\
    SqueezeNet v1.1~\cite{iandola2016squeezeNet} & 512 & 36.90 & 84.33 & 49.34 \\
    MobileNet-v2~\cite{sandler2018mobilenetv2} & 1280 & 36.67 & 83.96 & 49.11 \\
    DenseNet-201~\cite{huang2017densely} & 1920 & 37.72 & 85.42 & 50.33 \\
    GoogLeNet~\cite{szegedy2015going}  & 1024 & 37.55 & 84.96 & 50.10 \\
    \bottomrule
\end{tabular}
\vspace{-0.15in}
\end{table}

\subsection{Run-Time Analysis
}\label{subsec:run-time}

Fig.\,\ref{fig:RunTime} shows the measured runtime (in seconds) of our method compared to publicly available implementations of SMRS~\cite{elhamifar2012see} and SBOMP~\cite{mei2021patcha}.  
All experiments were conducted on a MacBook with an M1 chip, 16GB RAM, and macOS 15.5 using Matlab R2025a.  
Our method has two variants, with self-loops ($\beta=0$) and without ($\beta=2$), both showing comparable runtime.  
Total runtime consists of three components: graph construction, unfolding, and sampling.

For evaluation, we fixed the sampling budget to 10\% of the total frames.  
For each video length, the experiments were repeated 15 times, and we report the averaged runtime.  
The x-axis in Fig.~\ref{fig:RunTime} denotes the number of video frames, and the y-axis shows the measured runtime in seconds.

\begin{figure}[tb]
\begin{center}
\includegraphics[width=0.80\linewidth]{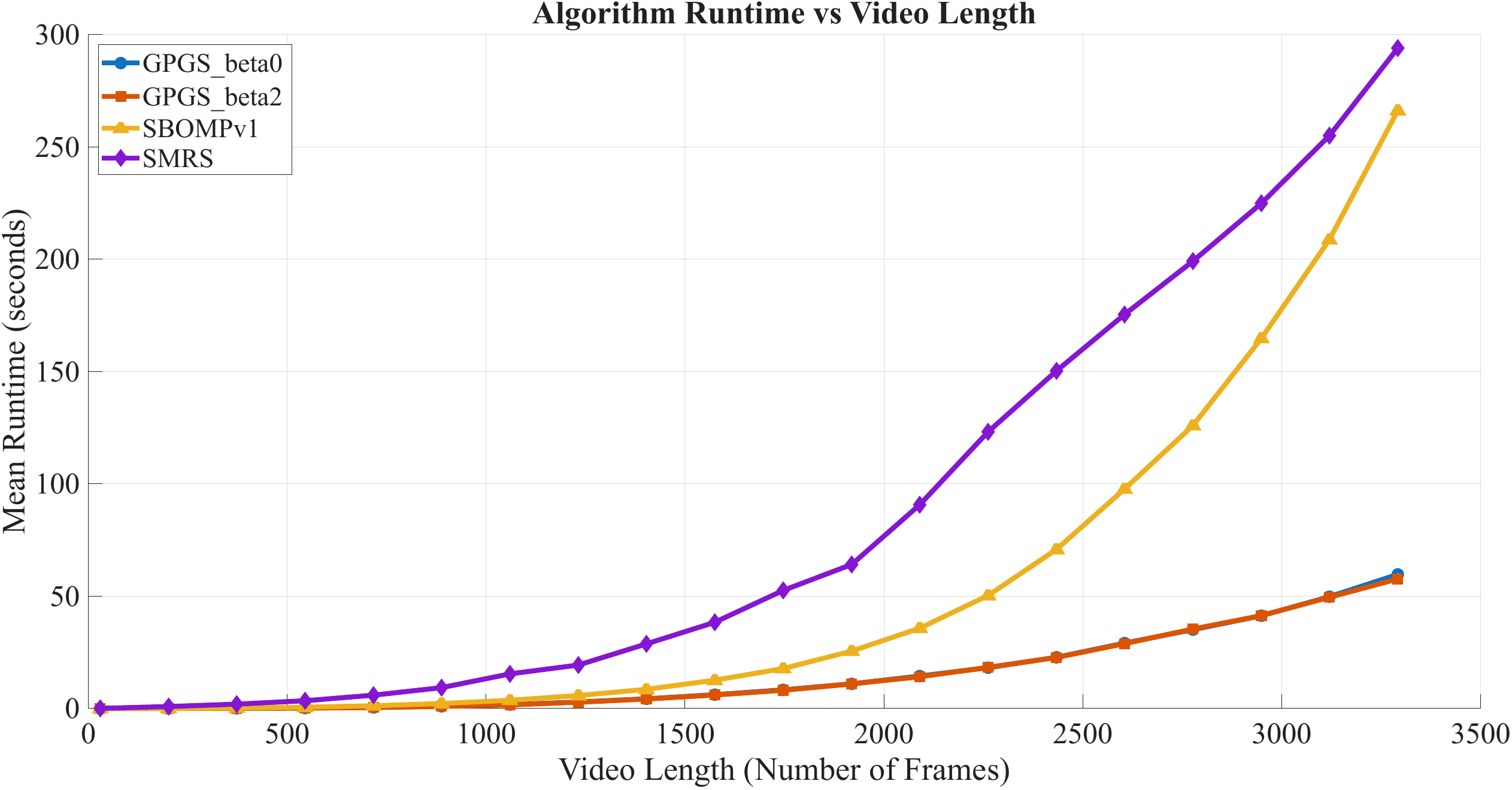}
\end{center}
\caption{Run‐time vs.\ number of frames, comparing our method to competitors.
}
\label{fig:RunTime}
\end{figure}

\subsection{Sensitivity Analysis}\label{subsec:sensitivity}

We evaluated robustness by sweeping $\mu$, $\sigma$, $\epsilon$, and the sampling budget over the ranges used in our experiments.
Fig.\,\ref{fig:sensitivity} shows Precision, Recall, and F$_1$ as each parameter varies while the remaining parameters are fixed to their tuned values.
Performance exhibits only modest variation within these reasonable ranges, indicating that the method is not highly sensitive to parameter tuning.
\begin{figure}[tbp]
\begin{center}
\includegraphics[width=0.98\linewidth]{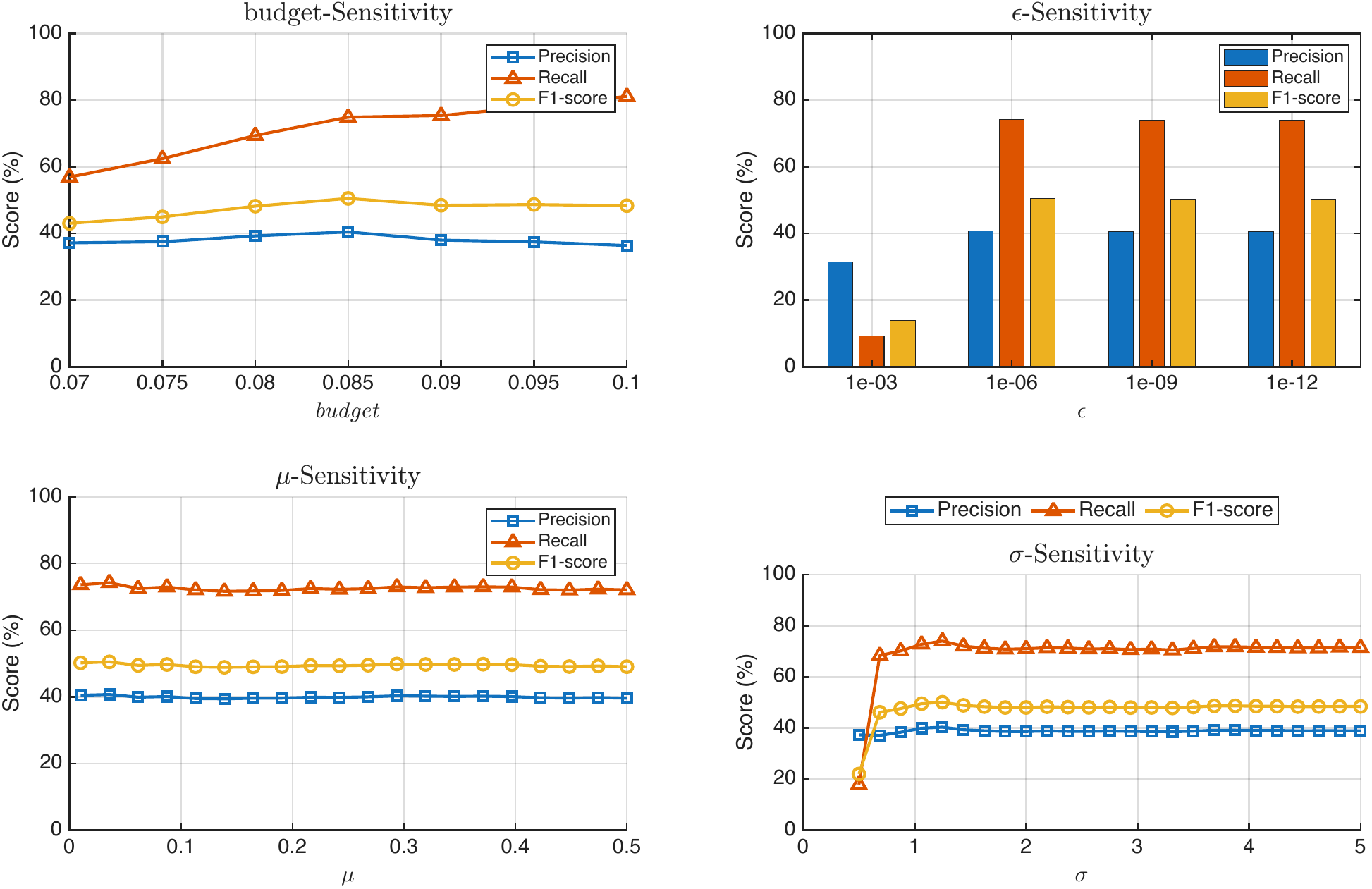}
		\caption{Sensitivity analysis of Precision, Recall, and F$_1$ versus $\mu$, $\sigma$, $\epsilon$, and budget on VSUMM.
		}\label{fig:sensitivity}
\end{center}
\end{figure}

\subsection{Statistical Significance}\label{subsec:statistical_significance}

We assess statistical significance on the MHSD dataset, where per-video results are available. 
Table~\ref{tab:statistical_significance} reports three complementary metrics: 99\% confidence intervals (CIs) of paired F$_1$-score differences (paired $t$-intervals), $p$-values from Wilcoxon signed-rank tests~\cite{hollander2013nonparametric}, and Cohen's $d$ effect sizes~\cite{cohen2013statistical}.
All pairwise comparisons yield $p < 10^{-12}$ and $d > 0.8$, indicating statistically significant improvements with large practical effects.
The entirely positive CIs confirm that our method consistently outperforms all baselines on MHSD.

\begin{table}[t]
\centering
\caption{Pairwise significance on MHSD (99\% CI, Wilcoxon $p$, Cohen's $d$); all comparisons are significant at $\alpha=0.01$ with large effects ($d>0.8$).}
\label{tab:statistical_significance}
\setlength{\tabcolsep}{4pt}
\begin{tabular}{lccc}
\toprule
Method Comparison & 99\% CI & $p$-value & Cohen's $d$ \\
\midrule
Ours vs.\ VSUMM~\cite{deavila2011vsumm}   & $[0.058, 0.104]$ & $3.60 \times 10^{-13}$ & $0.86$ \\
Ours vs.\ SMRS~\cite{elhamifar2012see}    & $[0.088, 0.148]$ & $2.30 \times 10^{-14}$ & $0.95$ \\
Ours vs.\ SBOMP~\cite{mei2021patcha}      & $[0.096, 0.150]$ & $2.91 \times 10^{-16}$ & $1.09$ \\
Ours vs.\ SSPVS~\cite{li2023progressivea} & $[0.070, 0.123]$ & $1.68 \times 10^{-13}$ & $0.88$ \\
Ours vs.\ CSTA~\cite{son2024cstaa}        & $[0.069, 0.127]$ & $1.80 \times 10^{-12}$ & $0.81$ \\
\bottomrule
\end{tabular}
\end{table}

\section{Conclusion}\label{sec:conclude}

We solve the keyframe selection problem to summarize a transitory video from a unique graph sampling perspective.  
Specifically, we first represent an $N$-frame transitory video as an $M$-hop path graph $\cG^o$, where $M \ll N$, and weight $w^o_{i,j}$ of an edge $(i,j)$ connecting two nodes (frames) within $M$ time instants is computed using feature distance between the corresponding feature vectors $\f_i$ and $\f_j$.
We unfold $\cG^o$ into a $1$-hop path graph $\cG$, specified by Laplacian $\cL$, via one of two analytically derived graph unfolding procedures to lower graph sampling complexity.  
Given graph $\cG$, we devise a linear-time algorithm that selects sample nodes specified by binary vector $\h$ to maximize the lower bound of the smallest eigenvalue $\lambda_{\min}^-(\B)$ for a coefficient matrix $\B = \diag{\h} + \mu \cL$.
Experiments show that our proposed method achieved comparable performance to SOTA methods at a reduced complexity.

\appendices
\section{Proof of Lemma~\ref{lemma:gctsl}}
\label{appendix:lemmaproof}

\begin{proof}
Define $\S^d = \diag{s_1^l,\ldots,s_{k-1}^l,s_k^u,\ldots,s_d^u}$, with demand scalars $s_i^l$ on the upstream nodes $i<k$ and supply scalars $s_i^u$ on the sample and downstream nodes $i \geq k$.
We verify that every Gershgorin disc left-end of $\S^d (\diag{\e_k} + \mu \cL^d)(\S^d)^{-1}$ lies at or beyond $T$.
Step $(a)$ below uses $s_i^l \geq 1$ at upstream non-sample nodes (the upstream stopping rule) and $s_j^u \geq 1$ at covered downstream nodes.

For sample node $k$, the disc left-end is
\begin{align}
& \cL_{k,k} + 1 - s_k^u \left( (s_{k-1}^l)^{-1} W_{k,k-1} + (s_{k+1}^u)^{-1} W_{k,k+1} \right)
\nonumber \\
\stackrel{(a)}{\geq}& \cL_{k,k} + 1 - s_k^u \left( (s_{k-1}^l)^{-1} W_{k,k-1} + W_{k,k+1} \right)
\stackrel{(b)}{=} T ,
\end{align}
where $(a)$ uses $s_{k+1}^u \geq 1$ and $(b)$ the definition of $s_k^u$ in \eqref{eq:su}.

For an upstream non-sample node $i \in \{2,\ldots,k-1\}$, the disc left-end is
\begin{align}
& \cL_{i,i} - s_i^l \left( (s_{i-1}^l)^{-1} W_{i,i-1} + (s_{i+1}^l)^{-1} W_{i,i+1} \right)
\nonumber \\
\stackrel{(a)}{\geq}& \cL_{i,i} - s_i^l \left( W_{i,i-1} + (s_{i+1}^l)^{-1} W_{i,i+1} \right)
\stackrel{(b)}{=} T ,
\end{align}
where $(a)$ uses $s_{i-1}^l \geq 1$ and $(b)$ the definition of $s_{i+1}^l$ in \eqref{eq:sl}.
The boundary $i = 1$, with $s_1^l = 1$ and $W_{1,0} = 0$, gives $\cL_{1,1} - (s_2^l)^{-1} W_{1,2} = T$.

The downstream nodes $\{k+1,\ldots,d\}$ follow by symmetry: supply scalars $s_i^u$ replace the demand scalars, and step $(a)$ drops the not-yet-expanded downstream neighbor via $s_{i+1}^u \geq 1$, so \eqref{eq:sd} pins the left-end at $T$ for each $i \in \{k+1,\ldots,d-1\}$.
At the boundary $i = d$ (likewise $d=k$), edge $(d,d+1)$ is absent from $\cL^d$, leaving left-end $\cL_{d,d} - s_d^u (s_{d-1}^u)^{-1} W_{d,d-1} = T + s_d^u W_{d,d+1} \geq T$.

Since all disc left-ends for $i \in \{1,\ldots,d\}$ lie at or beyond $T$, $\lambda^-_{\min}(\S^d (\diag{\e_k} + \mu \cL^d)(\S^d)^{-1}) \geq T$.
\end{proof}
 
\bibliographystyle{IEEEtran}

\end{document}